%% file: icml.tex
\documentclass[nohyperref]{article}

\usepackage{microtype}
\usepackage{graphicx}
\usepackage{subfigure}
\usepackage{booktabs} % for professional tables

\usepackage{hyperref}

\usepackage[accepted]{icml2022}

\usepackage{amsmath}
\usepackage{amssymb}
\usepackage{mathtools}
\usepackage{amsthm}

\usepackage[capitalize,noabbrev]{cleveref}

\usepackage[textsize=tiny]{todonotes}

\usepackage{shortcuts}
\usepackage{bbold}
\graphicspath{{images/},{prebuiltimages/}}

\icmltitlerunning{Stable Conformal Prediction Sets}

\begin{document}

\twocolumn[
% \icmltitle{Submission and Formatting Instructions for \\
%            International Conference on Machine Learning (ICML 2022)}

\icmltitle{Stable Conformal Prediction Sets}



\icmlsetsymbol{equal}{*}

\begin{icmlauthorlist}
\icmlauthor{Eugene Ndiaye}{aaa}
% \icmlauthor{Firstname2 Lastname2}{equal,yyy,comp}
% \icmlauthor{Firstname3 Lastname3}{comp}
\end{icmlauthorlist}

\icmlaffiliation{aaa}{H. Milton Stewart School of Industrial and Systems Engineering, Georgia Institute of Technology, Atlanta, GA, USA}

\icmlcorrespondingauthor{Eugene Ndiaye}{endiaye3@gatech.edu}

\icmlkeywords{Machine Learning, ICML}

\vskip 0.3in
]



\printAffiliationsAndNotice{}  % leave blank if no need to mention equal contribution

\begin{abstract}
When one observes a sequence of variables $(x_1, y_1), \ldots, (x_n, y_n)$, Conformal Prediction (CP) is a methodology that allows to estimate a confidence set for $y_{n+1}$ given $x_{n+1}$ by merely assuming that the distribution of the data is exchangeable. CP sets have guaranteed coverage for any finite population size $n$. While appealing, the computation of such a set turns out to be infeasible in general, \eg when the unknown variable $y_{n+1}$ is continuous. The bottleneck is that it is based on a procedure that readjusts a prediction model on data where we replace the unknown target by all its possible values in order to select the most probable one. This requires computing an infinite number of models, which often makes it intractable. In this paper, we combine CP techniques with classical algorithmic stability bounds to derive a prediction set computable with a single model fit. We demonstrate that our proposed confidence set does not lose any coverage guarantees while avoiding the need for data splitting as currently done in the literature. We provide some numerical experiments to illustrate the tightness of our estimation when the sample size is sufficiently large, on both synthetic and real datasets. \looseness=-1
\end{abstract}

\input{subfiles/introduction}
\input{subfiles/framework}

\input{subfiles/stability_bounds}

\input{subfiles/experiments}

\section*{Acknowledgements}
We warmly thank the reviewers for their insightful comments and contributions to improve the presentation of this paper. We also thank Elvis Dohmatob and Xiaoming Huo for proofreading and for pointing out mistakes in notations.\looseness=-1
\bibliography{references}
\bibliographystyle{icml2022}

\newpage
\onecolumn
\input{subfiles/appendix}

\end{document}

%% file: subfiles/introduction.tex
\section{Introduction}

Modern machine learning algorithms can predict the label of an object based on its observed characteristics with impressive accuracy. They are often trained on historical datasets sampled from the same distribution and it is important to quantify the uncertainty of their predictions. Conformal prediction is a versatile and simple method introduced in \citep{Vovk_Gammerman_Shafer05, Shafer_Vovk08} that provides a finite sample and distribution free $100(1 - \alpha)\%$ confidence region on the predicted object based on past observations. 
The main idea can be subsumed as a hypothesis testing between \looseness=-1
\begin{align*}
H_0:\, y_{n+1} = z \qquad \text{ and } \qquad  H_1:\, y_{n+1} \neq z \enspace,
\end{align*}
where $z$ is any replacement candidate for the unknown response $y_{n+1}$. The conformal prediction set will consist of the collection of candidates whose tests are not rejected. The construction of a $p$-value function is simple.  We start by fitting a model with training set $\{(x_1, y_1), \ldots, (x_n, y_n), (x_{n+1}, z)\}$ and sort the prediction scores/errors for each instance in ascending order. A candidate $z$ will be considered as conformal or typical if the rank of its score is sufficiently small compared to the others. The key assumption is that the predictive model and the joint probability distribution of the sequence $\{(x_i, y_i)\}_{i=1}^{n+1}$ are invariant \wrt permutation of the data. As a consequence, the ranks of the scores are equally likely and thus follow a uniform distribution which allow to calibrate a threshold on the rank statistics leading to a valid confidence set. This method has a strong coverage guarantee without any further assumptions on the distribution and is valid for any finite sample size $n$; see more details in \Cref{sec:Conformal_Prediction}.\looseness=-1

Conformal prediction technique has been applied for designing uncertainty sets in active learning \citep{Ho_Wechsler08}, anomaly detection \citep{Laxhammar_Falkman15, Bates_Candes_Lei_Romano_Sesia21}, few shot learning \citep{Fisch_Schuster_Jaakkola_Barzilay21}, time series \citep{Chernozhukov_Wuthrich_Zhu18, Xu_Xie20, Chernozhukov_Wuthrich_Zhu21}, robust optimization \citep{Johnstone_Cox21} or to infer the performance guarantee for statistical learning algorithms \citep{Holland20, Cella_Martin20}. Currently, we are seeing a growing interest in these approaches due to their flexibility and ease of deployment even for very complex problems where classical approaches offer limited performance \citep{Efron21}. We refer to \citep{Balasubramanian_Ho_Vovk14} for other AI applications.\looseness=-1

Despite its nice properties, the computation of conformal prediction sets requires fitting a model on a new augmented dataset where the unknown quantity $y_{n+1}$ is replaced by a set of candidates. In a regression setting where an object can take an uncountable possible value, the set of candidates is infinite. Therefore, computing the conformal prediction is infeasible without additional structural assumptions about the underlying model fit, and even so, the current computational costs remain very high. Hence the prevailing recommendation to use less efficient data splitting methods. \looseness=-1

\paragraph{Contribution.} We leverage algorithmic stability to bound the variation of the predictive model \wrt to changes in the input data. This results in a circumvention of the computational bottleneck induced by the necessary readjustment of the model each time we want to assess the typicalness of a candidate replacement of the target variable. As such, we can provide a tight estimation of the confidence sets without loss in the coverage guarantee. Our method is computationally and statistically efficient since it requires only a single model fit and does not involve any data splitting.\looseness=-1

\paragraph{Notation.}
For a nonzero integer $n$, we denote $[n]$ to be the set $\{1, \cdots, n\}$. The dataset of size $n$ is denoted $\Data_n = (x_i, y_i)_{i \in [n]}$, the row-wise input feature matrix $X = [x_1, \cdots, x_n, x_{n+1}]^\top$. Given a set $\{u_1, \cdots, u_n\}$, the rank of $u_j$ for $j \in [n]$ is defined as $$\mathrm{Rank}(u_{j}) = \sum_{i=1}^{n} \mathbb{1}_{u_i \leq u_{j}} \enspace.$$ We denote $u_{(i)}$ the $i$-th order statistics.

%% file: subfiles/framework.tex
\section{Conformal Prediction}\label{sec:Conformal_Prediction}

Conformal prediction \citep{Vovk_Gammerman_Shafer05} is a framework for constructing online confidence sets, with the remarkable properties of being distribution free, having a finite sample coverage guarantee, and being able to be adapted to any estimator under mild assumptions. We recall the arguments in \citep{Shafer_Vovk08, Lei_GSell_Rinaldo_Tibshirani_Wasserman18} to construct a conformity/typicalness function based on rank statistics that yields to distribution-free inference methods. The main tool is that the rank of one variable among an exchangeable and identically distributed sequence follows a (sub)-uniform distribution \citep{Brocker_Kantz11}.\looseness=-1
\begin{lemma}\label{lm:distribution_of_rank}
Let $U_1, \ldots, U_n, U_{n+1}$ be an exchangeable and identically distributed sequence of random variables. Then for any $\alpha \in (0, 1)$, we have 
$$\mathbb{P}^{n+1}(\mathrm{Rank}(U_{n+1}) \leq (n+1)(1 - \alpha)) \geq 1 - \alpha \enspace.$$
\end{lemma} 

We remind that $y_{n+1}$ is the \emph{unknown} target variable. We introduce a learning problem with the augmented training data $\Data_{n+1}(z) := \Data_{n} \cup \{(x_{n+1}, z)\}$ for $z \in \bbR$ and with the augmented vector of labels $y(z) = (y_1, \cdots, y_n, z)$:
\begin{equation}\label{eq:model_optimization}
\beta(z) \in \argmin_{\beta \in \bbR^p} \mathcal{L}(y(z), \Phi(X, \beta)) + \Omega(\beta) \enspace,
\end{equation}
where $\Phi$ is a feature map and for any parameter $\beta \in \bbR^p$ $$\Phi(X, \beta) = [\Phi(x_1, \beta), \ldots, \Phi(x_{n+1}, \beta)] \in \bbR^{n+1} \enspace.$$ 
Given an input feature vector $x$, the prediction of its output/label adjusted on the augmented data, can be defined as\looseness=-1 $$\mu_z(x) := \Phi(x, \beta(z)) \enspace.$$
For example in case of empirical risk minimization, we have $$\mathcal{L}(y(z), \Phi(X, \beta)) = \sum_{i=1}^{n} \ell(y_i, \Phi(x_{i}, \beta)) + \ell(z, \Phi(x_{n+1}, \beta)) \enspace.$$
There are many examples of cost functions in the literature. A popular example is the power norm regression, where $\ell(a, b) = |a - b|^q$. When $q=2$, this corresponds to the classical linear regression. The cases where $q = (1, 2)$ are frequent in robust statistics where the case $q = 1$ is known as the least absolute deviation. The loss \texttt{logcosh} $\ell(a, b) = \gamma\log(\cosh(a - b)/\gamma)$ is a differentiable alternative to the $\ell_{\infty}$ norm (Chebychev approximation). One can also have the loss function \texttt{Linex} \cite{Gruber10, Chang_Hung07} which provides an asymmetric loss function $\ell(a, b) = \exp(\gamma(a - b)) - \gamma(a - b) - 1$, for $\gamma \neq 0$. Any convex regularization function $\Omega$ \eg Ridge \cite{Hoerl_Kennard70} or norm inducing sparsity \cite{Bach_Jenatton_Mairal_Obozinski12} can be considered. Also the feature map $\Phi$ can be parameterized and learned \textit{\`a la} neural network. \Cref{eq:model_optimization} includes many modern formulations of statistical learning estimators. The only requirement on these is to be invariant with respect to the data permutation; this leaves a very large degree of freedom on their choice. For example, $\beta(z)$ can be the output of an iterative model \eg proximal gradient descent, with early stopping. \looseness=-1

Let us define the conformity measure for $\Data_{n+1}(z)$ as
\begin{align}\label{eq:arbitrary_conformity_measure}
\forall i \in [n],\, E_{i}(z) &= S(y_i, \mu_z(x_{i})) \enspace,\\
E_{n+1}(z) &= S(z, \mu_z(x_{n+1})) \enspace,
\end{align}
where $S$ is a real-valued function \eg in a linear regression problem, one can take $s(a, b) = |a - b|$. The main idea for constructing a conformal confidence set is to consider the typicalness/conformity of a candidate point $z$ measured as
\begin{equation}\label{eq:general_def_of_pi}
\pi(z) := 1 - \frac{1}{n+1} \mathrm{Rank}(E_{n+1}(z)) \enspace.
\end{equation}
The conformal set gathers all the real values $z$ such that $\pi(z) \geq \alpha$, if and only if, the score $E_{n+1}(z)$ is ranked no higher than $\lceil(n+1)(1 - \alpha)\rceil$, among $\{E_{i}(z)\}_{i \in [n + 1]}$ \ie
\begin{equation}\label{eq:exact_conformal_set}
\Gamma^{(\alpha)}(x_{n+1}) := \{z \in \bbR:\, \pi(z) \geq \alpha \} \enspace.
\end{equation}

\begin{figure}
  \centering
  \includegraphics[width=\columnwidth]{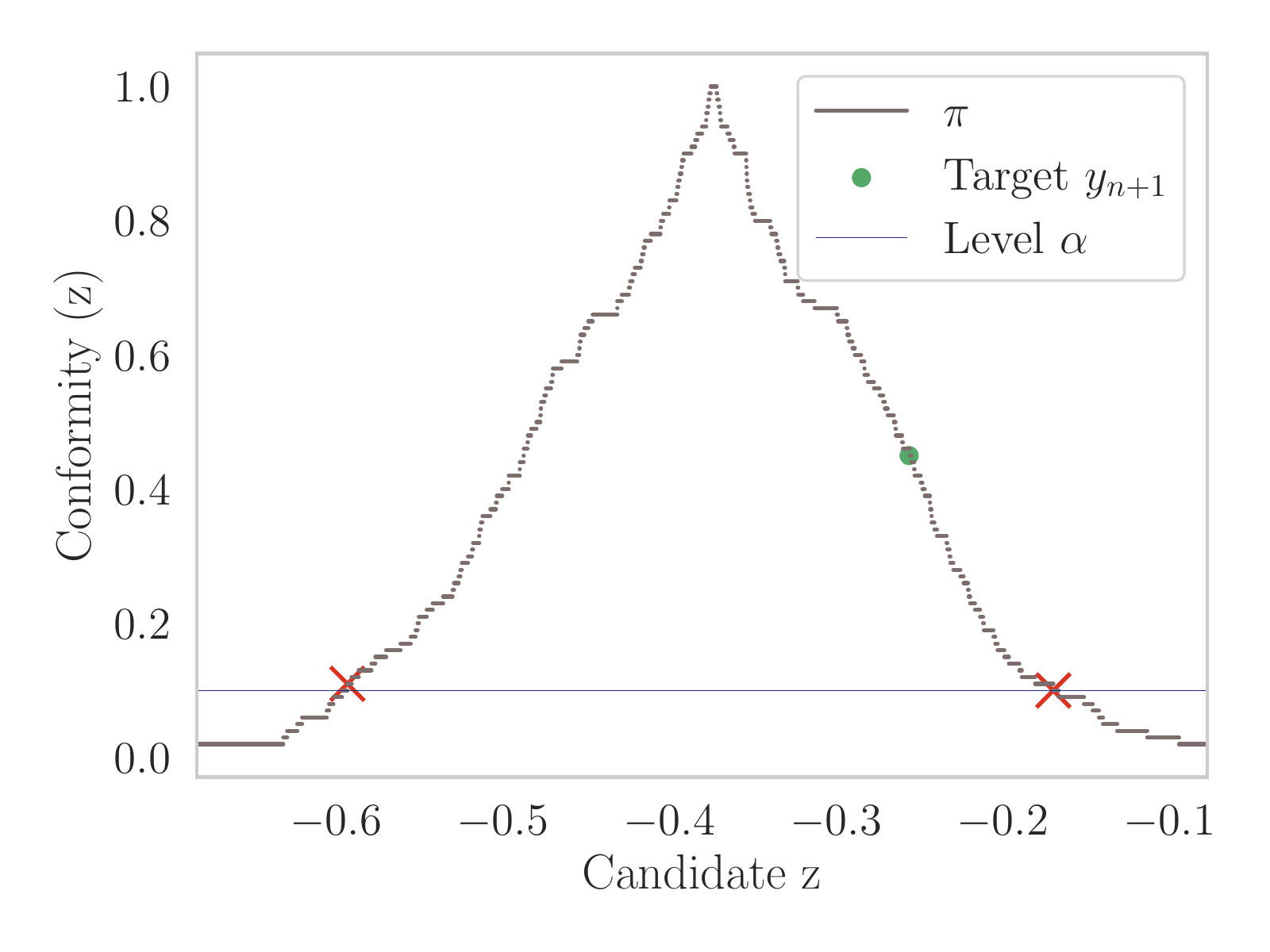}
  \caption{Illustration of a conformal prediction set with a confidence level level $0.9$ \ie $\alpha=0.1$. The ends of the CP set are indicated by the red cross. \label{fig:illustration_cp}}
\end{figure}

A direct application of \Cref{lm:distribution_of_rank} to $U_i = E_i(y_{n+1})$ reads $\mathbb{P}(\pi(y_{n+1}) \leq \alpha) \leq \alpha$ \ie the random variable $\pi(y_{n+1})$ takes small values with small probability and it reads the coverage guarantee
$$ \mathbb{P}(y_{n+1} \in \Gamma^{(\alpha)}(x_{n+1})) \geq 1 - \alpha \enspace.$$

\Cref{fig:illustration_cp} illustrates the candidates selected for inclusion in the confidence set as the most likely variables.

\subsection{Computational Limitations and Previous Works}

For regression problems where $y_{n+1}$ lies in a subset of $\bbR$, obtaining the conformal set $\Gamma^{(\alpha)}(x_{n+1})$ in \Cref{eq:exact_conformal_set} is computationally challenging. It requires re-fitting the prediction model $\beta(z)$ for infinitely many candidates $z$ in order to compute the map of conformity measure such as $z \mapsto E_{i}(z) = |y_i - x_{i}^{\top}\beta(z)|$. Except for a few examples, the computation of a conformal prediction set is infeasible in general. We describe below some successful computational strategies while pointing out their potential shortcomings.

In Ridge regression, for any $x$ in $\bbR^p$, $z \mapsto x^{\top}\beta(z)$ is a linear function of $z$, implying that $E_{i}(z)$ is piecewise linear. Exploiting this fact, an exact conformal set $\Gamma^{(\alpha)}(x_{n+1})$ for Ridge regression was efficiently constructed in \citep{Nouretdinov_Melluish_Vovk01}.
Similarly, using the piecewise linearity \wrt sparsity level of the Lasso path provided by the \texttt{Lars} algorithm \citep{Efron_Hastie_Johnstone_Tibshirani04}, \citep{Hebiri10} builds a sequence of conformal sets for the Lasso associated to the transition points of the \texttt{Lars} with the observed data $\Data_n$. Nevertheless, such procedure breaks the proof technique for the coverage guarantee as the exchangeability of the sequence $(E_{i}(y_{n+1}))_{i \in [n+1]}$ is not necessarily maintained. However, a slight adaptation can fix the previous problem. Indeed using the piecewise linearity in $z$ of the Lasso solution, \citep{Lei19} proposed a piecewise linear homotopy under mild assumptions, when a single input sample point is perturbed. This finally allows to compute the whole solution path $z \mapsto \beta(z)$ and successfully provides a conformal set for the Lasso and Elastic Net. These processes are however limited to quadratic loss function. Later, \citep{Ndiaye_Takeuchi19} proposed an adaptation using approximate solution path \citep{Ndiaye_Le_Fercoq_Salmon_Takeuchi2019} instead of exact solution. This results in a careful discretization of the set of candidates restricted into a preselected compact $[z_{\min}, z_{\max}]$. Assuming that the optimization problem in \Cref{eq:model_optimization} is convex and that the loss function is smooth, this leads to a computational complexity of $O(1/\sqrt{\epsilon})$ where $\epsilon>0$ is a prescribed optimization error. All these previous methods are at best restricted to convex optimization formulations. A different road consists in merely assuming that the conformal set $\Gamma^{(\alpha)}(x_{n+1})$ in \Cref{eq:exact_conformal_set} itself is a bounded interval. As such, its endpoints can be estimated by approximating the roots of the function $z \mapsto \pi(z) - \alpha$. Under slight additional assumptions, a direct bisection search can then compute a conformal set with a complexity of $O(\log_2(1/\epsilon_r))$ \citep{Ndiaye_Takeuchi21} where $\epsilon_r>0$ is the tolerance error \wrt to exact root. \looseness=-1

Cross-conformal Predictors was initially introduced in its one split version in \citep{Papadopoulos_Proedrou_Vovk_Gammerman02}.The idea is to separate the data into two independent parts, fit the model on one part and rank the scores on the other part where \Cref{lm:distribution_of_rank} remains applicable and thus preserves the coverage guarantee. Although this approach avoids the computational bottleneck by requiring only one data adjustment, the statistical efficiency of the model may be reduced due to a much smaller sample size available during the training and calibration phases. In general, the proportion of the training set to the calibration set is a hyperparameter that requires appropriate tuning: a small calibration set leads to highly variable conformational scores and a small training set leads to poor model fit. Such trade-off is very recurrent in machine learning and often appears in the debate between bias reduction and variance reduction. It is often decided by the cross-validation method with several folds \citep{Arlot_Celisse10}. \textit{Cross-conformal predictors} \citep{Vovk15} follow the same ideas and exploit the full dataset for calibration and significant proportions for training the model. The dataset is partitioned into $K$ folds and one performs a split conformal set by sequentially defining the $k$th fold as calibration set and the remaining as training set for $k \in \{1, \ldots, K\}$. The leave-one-out aka Jackknife CP set, requires $K=n$ model fit which is prohibitive even when $n$ is moderately large. On the other hand, the $K$-fold version will require $K$ model fit but will come at the cost of fitting on a lower sample size and will leads to an additional excess coverage of $O(\sqrt{2/n})$ and requires a subtle aggregation of the different pi-values obtained; see \citep{Carlsson_Eklund_Norinder14, Linusson_Norinder_Bostrom_Johansson_Lofstrom17}. \citep{Barber_Candes_Ramdas_Tibshirani21} shown that the confidence level attained is $1 - 2\alpha$ instead of $1 - \alpha$ and can only \textit{approximately} reaches the target coverage $1 - \alpha$ under additional stability assumption.\looseness=-1 

Although these recent advances have drastically improved the tractability of the calculations, in practice multiple re-adjustments of the data are required. This remains very expensive especially for complex models. Imagine having to re-train a neural network from scratch ten or twenty times to get a reasonable estimate. In this paper, we actually show that a single model fit is enough to tightly approximate the conformal set when the underlying model fitting is stable.\looseness=-1

\section{Approximation via Algorithmic Stability}

\begin{figure*}[t]
  \centering
  \subfigure[$n=30$]{\includegraphics[width=0.33\textwidth]{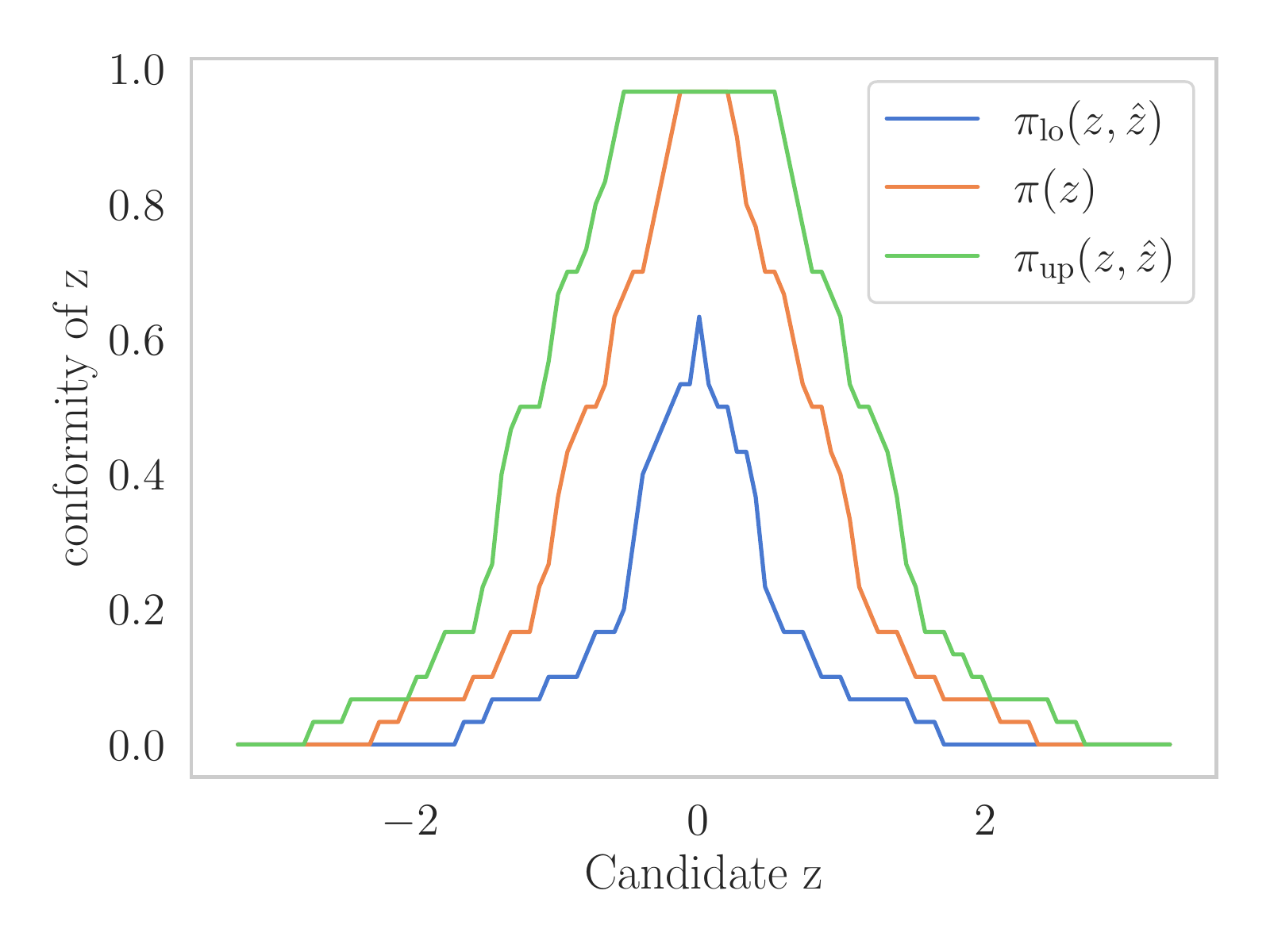}}
  \subfigure[$n=90$]{\includegraphics[width=0.33\textwidth]{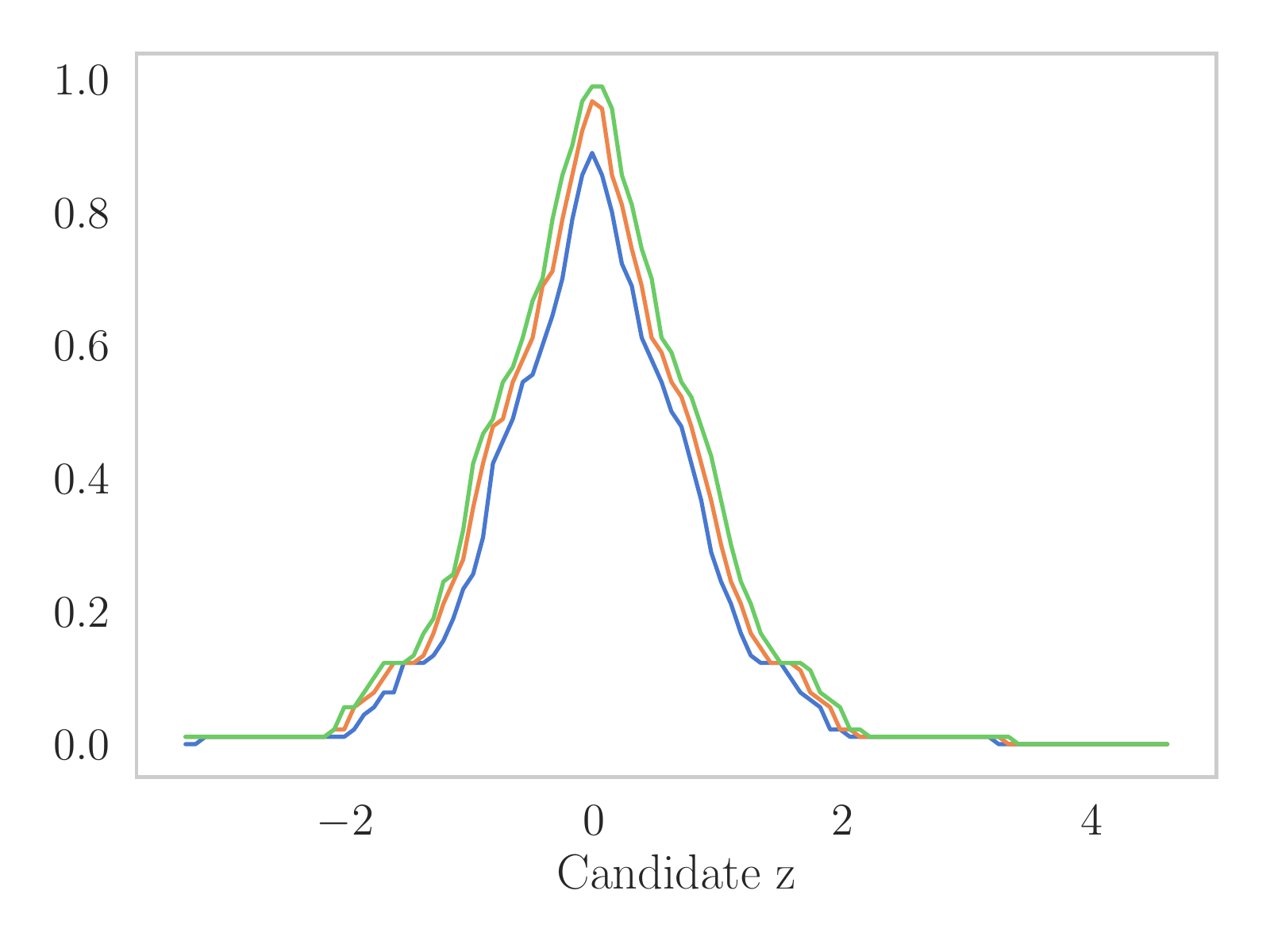}}
  \subfigure[$n=300$]{\includegraphics[width=0.33\textwidth]{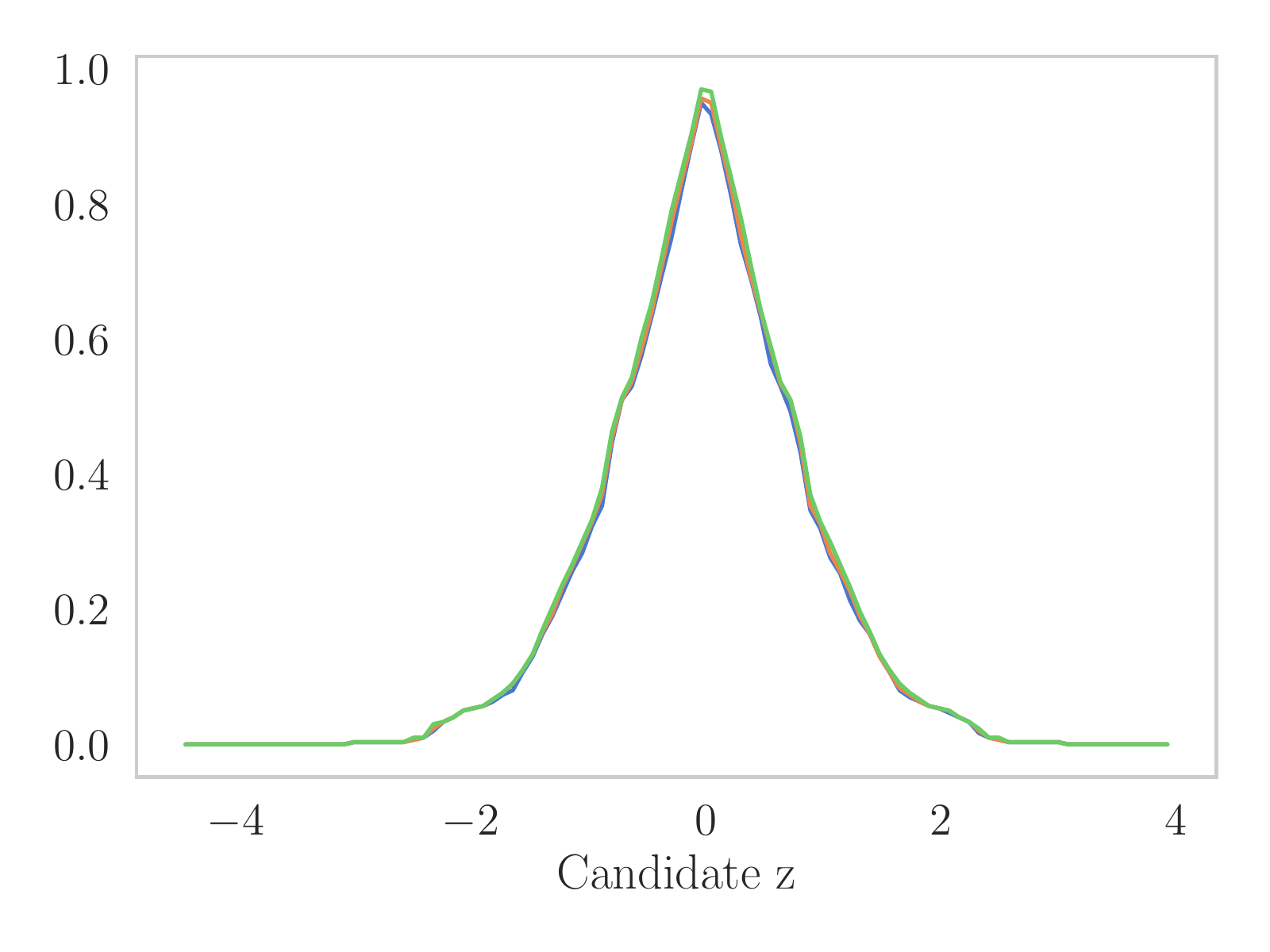}}
  \subfigure[Convergence of the approximation gap $\rm{Gap}(z, \hat z)$]{\includegraphics[width=0.33\textwidth]{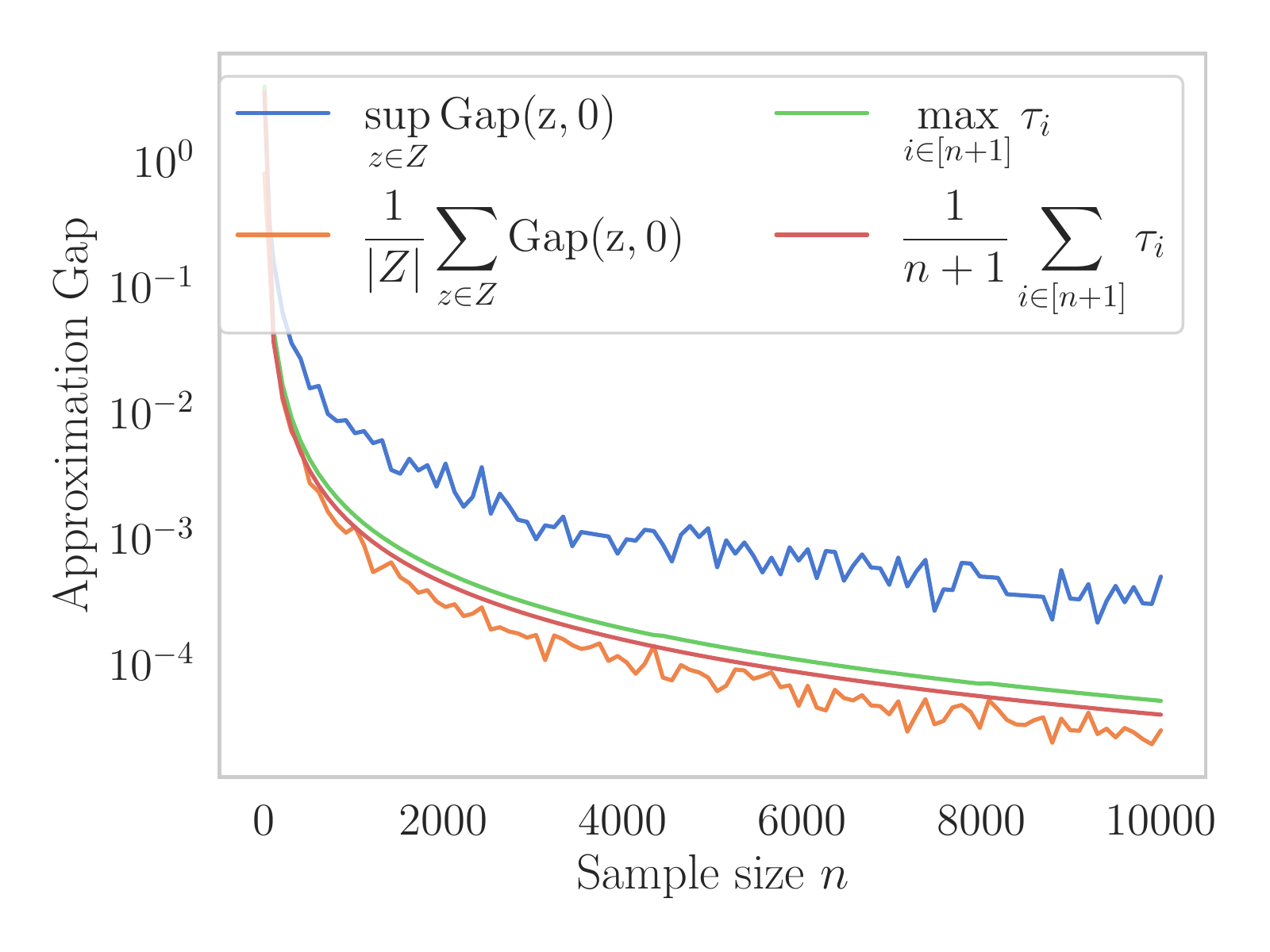}}
\subfigure[\texttt{Diabetes} $(442, 10)$]{\includegraphics[width=0.33\textwidth]{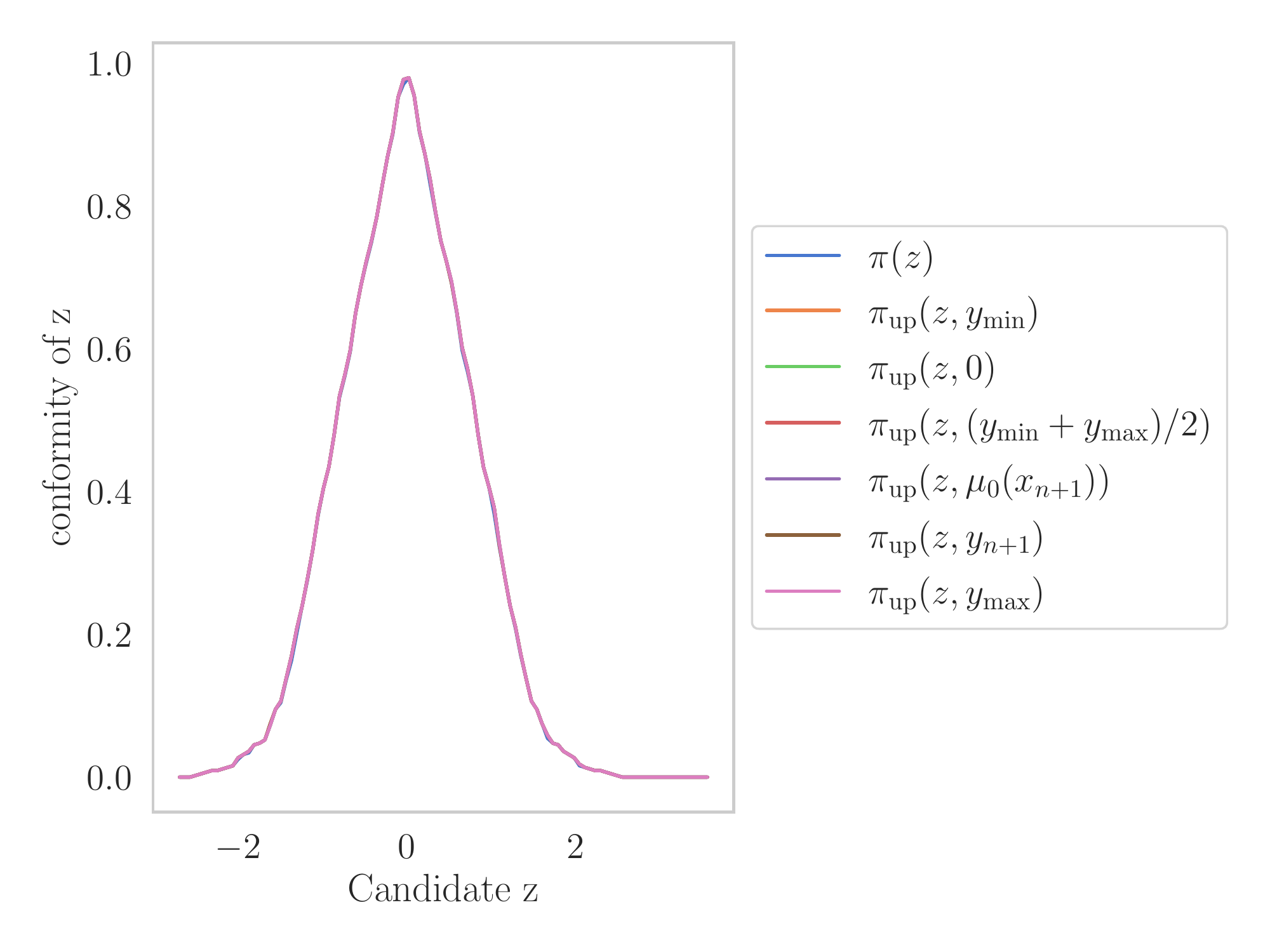}}
\subfigure[\texttt{Synthetic} $(30, 100)$]{\includegraphics[width=0.33\textwidth]{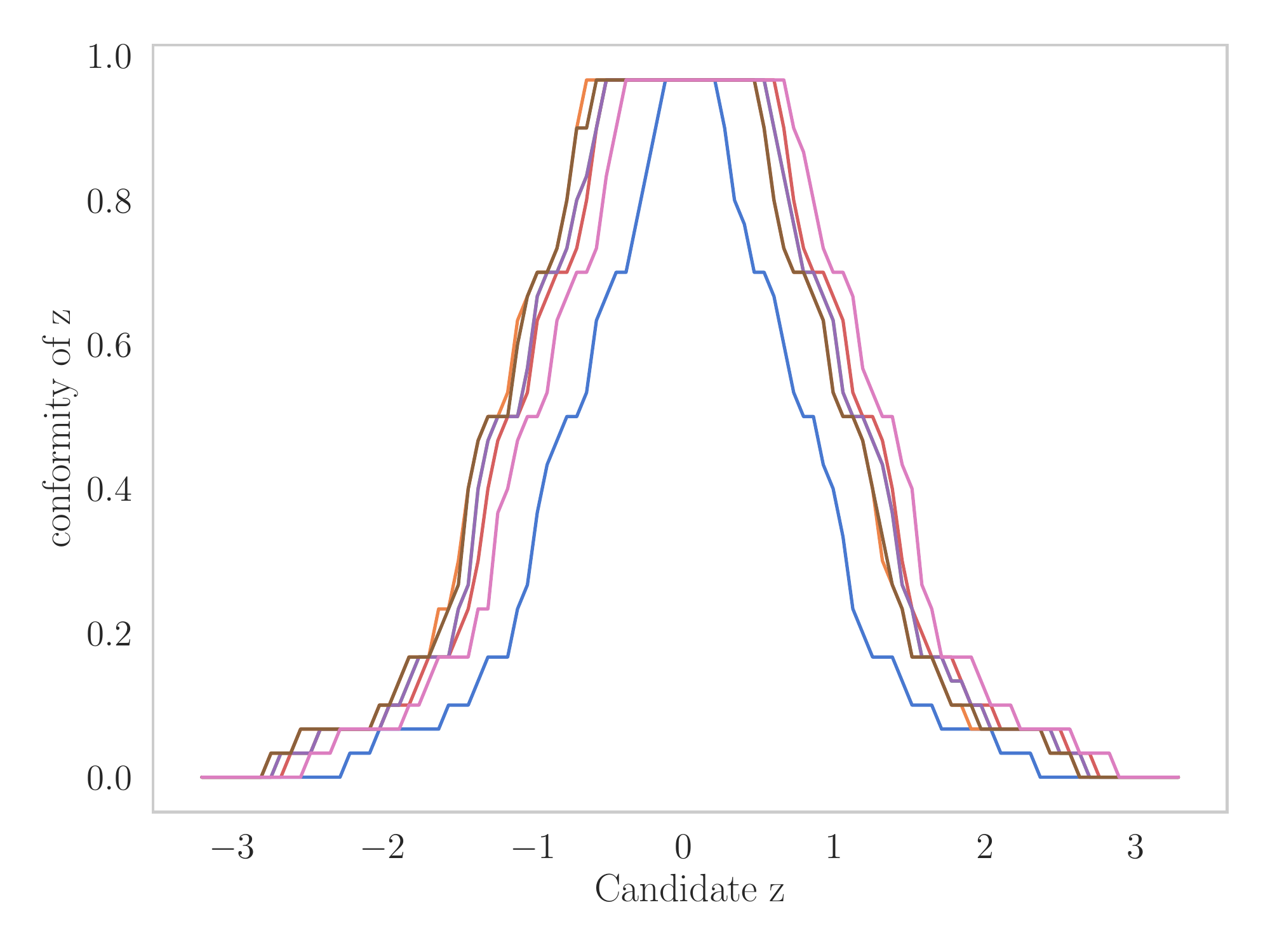}}
\caption{Illustration of the evolution of the conformity function as a function of sample size. The underlying model fit 
is $\beta(z) \in \argmin_{\beta \in \bbR^p} \norm{y(z) - X\beta}_1 / (n + 1) + \lambda\norm{\beta}^2$ where $y(z) = (y_1, \cdots, y_n, z)$ and 
we use \texttt{sklearn} synthetic dataset $\rm{make\_regression}(n, p=100, \rm{noise}=1)$. We fixed $\hat z = 0$ and $\lambda=0.5$. The set $Z$ is a linear grid in the interval $[y_{(1)}, y_{(n)}]$. We also illustrate the batch approximation for different values of $\hat z$. \label{fig:illustration_stab}}
% \caption{Illustration of the approximation for different values of $\hat z$. We use the same setting as in \Cref{fig:illustration_stab} \label{fig:illustration_stab_different_base}}
\end{figure*}

The \Cref{sec:Conformal_Prediction} guarantees that $\pi(y_{n+1}) \geq \alpha$ with high probability. Therefore, since $y_{n+1}$ is unknown, the conformal set just selects all $z$ that satisfies the same inequality \ie $\Gamma^{(\alpha)}(x_{n+1}) = \{z : \pi(z) \geq \alpha\}$. This leads to fitting a new model for any $z$. Here, we take a different strategy. The main remark is that only one element of the dataset changes at a time, then with mild stability assumptions, one can expect that the model prediction will not change drastically. Instead of inverting $\pi(\cdot)$, we will bound it with quantities independent of the model fit $\mu_{z}$ for any $z$.
 
% Thus let us assume that we use the score $E_i(z) = |y_i - \mu_z(x_i)|$ where $\mu_z$ is the prediction function and that the following stability bound holds:
% \begin{equation}
% |\mu_{z}(x_i) - \mu_{z_0}(x_i)| \leq \tau_i \quad \forall z, z_0 \enspace.
% \end{equation}
% From the triangle inequality, we have for any $q, z, \hat z$ and $i \in [n + 1]$
% $$
% \left| |q - \mu_{z}(x_i)| - |q - \mu_{\hat z}(x_i)| \right| \leq 
% |\mu_{z}(x_i) - \mu_{\hat z}(x_i)| \leq \tau_i.
% $$
% 
% Let $\mu_z$ be the prediction function and that the following stability bound holds \ie
% \begin{definition}[Algorithmic Stability]
% A prediction function $\mu_{\cdot}$ is stable if for any observed features $x_i$, $i \in [n+1]$, we have
% \begin{equation}
% |\mu_{z}(x_i) - \mu_{z_0}(x_i)| \leq \tau_i \quad \forall z, z_0 \enspace.
% \end{equation}
% \end{definition}

\begin{definition}[Algorithmic Stability]\label{def:Algorithmic_Stability}
A prediction function $\mu_{\cdot}$ is stable if for any observed features $x_i$, $i \in [n+1]$, we have
\begin{equation}
|S(q, \mu_{z}(x_i)) - S(q,\mu_{z_0}(x_i))| \leq \tau_i \quad \forall z, z_0, q \in \bbR \enspace.
\end{equation}
\end{definition}

In the literature, it is common to make assumptions about the stability of a predictive model to obtain upper bounds on its generalization error and thus ensure that it does not overfit the training data \eg \citep{Bousquet_Elisseeff02}. Although the conformal prediction framework applies even when the underlying model is not stable, we show that this additional assumption allows for efficient evaluation of confidence sets. We also add that in cases where the generalization capabilities of the model are poor, the size of the confidence intervals can become very large; even unbounded, and not at all informative.

\begin{proposition}\label{prop:bounding_pi} Assume that the model fit $\mu_{\cdot}$ is stable as in \Cref{def:Algorithmic_Stability}. Then, we have:
$$\forall z, \hat z,\quad \pi_{\rm{lo}}(z, \hat z) \leq \pi(z) \leq \pi_{\rm{up}}(z, \hat z) \enspace,$$
with
\begin{align*}
\pi_{\rm{lo}}(z, \hat z) &:= 1 - \frac{1}{n+1} \sum_{i=1}^{n+1} \mathbb{1}_{L_i(z, \hat z) \leq U_{n+1}(z, \hat z)} \enspace,\\
\pi_{\rm{up}}(z, \hat z) &:= 1 - \frac{1}{n+1} \sum_{i=1}^{n+1} \mathbb{1}_{U_i(z, \hat z) \leq L_{n+1}(z, \hat z)} \enspace,
\end{align*}
where, we define, for any index $i$ in $[n]$,
\begin{align*}
L_i(z, \hat z) &= E_i(\hat z) - \tau_i \enspace, \\ 
U_i(z, \hat z) &= E_i(\hat z) + \tau_i \enspace, \\
L_{n+1}(z, \hat z) &= S(z, \mu_{\hat z}(x_{n+1})) - \tau_{n+1} \enspace, \\
U_{n+1}(z, \hat z) &= S(z, \mu_{\hat z}(x_{n+1})) + \tau_{n+1} \enspace.
\end{align*}
\end{proposition}

\begin{proof}
By stability, for any $q$, we have:
$$ |S(q, \mu_{z}(x_i)) - S(q, \mu_{\hat z}(x_i))| \leq  \tau_i \enspace.$$
Applying the previous inequality to $q=y_i$ for any index $i$ in $[n+1]$, we have
$L_i(z, \hat z) \leq E_i(z) \leq U_i(z, \hat z)$ and it holds:
\begin{align*}
U_i(z, \hat z) \leq L_{n+1}(z, \hat z) &\Longrightarrow
E _i(z) \leq E_{n+1}(z) \\
&\Longrightarrow L_i(z, \hat z) \leq U_{n+1}(z, \hat z) \enspace.
\end{align*}
Taking the indicator of the corresponding sets, we obtain the result.
\end{proof}

A direct consequence of \Cref{prop:bounding_pi} is that the exact conformal set can be wrapped as follows.

\begin{corollary}[Stable Conformal Sets] Under the assumption of \Cref{prop:bounding_pi}, the conformal prediction set is lower and upper approximated as 
$$\Gamma_{\rm{lo}}^{(\alpha)}(x_{n+1}) \subset \Gamma^{(\alpha)}(x_{n+1}) \subset\Gamma_{\rm{up}}^{(\alpha)}(x_{n+1}) \enspace,$$
where
\begin{align*}
\Gamma_{\rm{lo}}^{(\alpha)}(x_{n+1}) &= \{z: \pi_{\rm{lo}}(z, \hat z) \geq \alpha \} \enspace,\\
\Gamma_{\rm{up}}^{(\alpha)}(x_{n+1}) &= \{z: \pi_{\rm{up}}(z, \hat z) \geq \alpha \} \enspace.
\end{align*}
\end{corollary}

Since our proposal arises from a combination of the conformal prediction sets with a correction from the stability bounds, we call the resulting (upper) confidence set $\Gamma_{\rm{up}}^{(\alpha)}(x_{n+1})$ \textbf{\texttt{stabCP}} for stable conformal set. By construction, it contains the exact confidence set $\Gamma^{(\alpha)}(x_{n+1})$ and therefore enjoys at least the same statistical benefits displayed in the following result.
\begin{proposition}[Coverage guarantee]
Assume that the model fit $\mu_{\cdot}$ is stable as in \Cref{def:Algorithmic_Stability}. Then the \texttt{stabCP} set is an upper envelope of the exact conformal prediction set in \Cref{eq:exact_conformal_set} and is thus valid \ie
$$ \mathbb{P}(y_{n+1} \in \Gamma_{\rm{up}}^{(\alpha)}(x_{n+1})) \geq 1 - \alpha \enspace. $$ 
\end{proposition}
As promised in the abstract, our proposed method suffers no loss of statistical coverage,  requires only one model adjustment to the data at an arbitrary candidate point $\hat z$, and fully uses all the data (no splitting). Thus we can benefit both from statistical efficiency with a smaller confidence interval as in the case of the exact calculation; but also we completely break the computational difficulty as in the case of splitting methods. To our knowledge, there is no equivalent method that can benefit from such a double performance.

\subsection{Practical Computation of \texttt{stabCP} sets}

By construction, the computation of stable conformal sets is equivalent to collecting all $z$ such that $\pi_{\rm{up}}(z, \hat z) \geq \alpha$. Let's begin by noting that $U_{n+1}(z, \hat z) > L_{n+1}(z, \hat z)$ when $\tau_{n+1} > 0$ which we will assume for simplicity. We have
$$ \pi_{\rm{up}}(z, \hat z) \geq \alpha \Leftrightarrow \sum_{i=1}^{n} \mathbb{1}_{U_i(z, \hat z) \leq L_{n+1}(z, \hat z)} \leq (1 - \alpha)(n+1) \enspace.$$
This means that a candidate $z$ is selected, if at most $(1-\alpha)(n+1)$ elements of $\{U_i(z, \hat z)\}_{i \in [n]}$ are smaller than $L_{n+1}(z, \hat z)$. Which is equivalent to\footnote{For $i \in [n]$, $U_i(z, \hat z)$ and $L_i(z, \hat z)$ are independent of $z$.}
$$ L_{n+1}(z, \hat z) \leq U_{(\lceil (1-\alpha)(n+1) \rceil)}(z, \hat z) =: Q_{1-\alpha}(\hat z) \enspace.$$
Hence, we can conclude that
$$ \Gamma_{\rm{up}}^{(\alpha)}(x_{n+1}) = \{z: S(z, \mu_{\hat z}(x_{n+1})) \leq Q_{1-\alpha}(\hat z) + \tau_{n+1}\}.$$
For the absolute value score, it reduces to the interval
$$ \Gamma_{\rm{up}}^{(\alpha)}(x_{n+1}) = [\mu_{\hat z}(x_{n+1}) \pm (Q_{1-\alpha}(\hat z) + \tau_{n+1})] \enspace.$$
For the sake of clarity, we summarize the computations for this simplest case in \Cref{alg:stabCP} and discuss the generalization in the appendix.
In general terms, \texttt{stabCP} sets are convex sets when the score function $z \mapsto S(z, \mu_{\hat z}(x_{n+1}))$ has convex level sets. This presumes that our strategy will also facilitate the calculations in cases where the target $y_{n+1}$ is multi-dimensional.

\begin{algorithm}[tb]
   \caption{Stable Conformal Prediction Set}
   \label{alg:stabCP}
\begin{algorithmic}
   \STATE {\bfseries Input:} data $\{(x_1, y_1), \ldots, (x_n, y_n)\}$ and $x_{n+1}$
   \STATE Coverage level $\alpha \in (0, 1)$, any estimate $\hat z \in \bbR$
   \STATE Stability bounds $\tau_1, \ldots, \tau_{n+1}$ of the learning algorithm
   \STATE {\bfseries Output:} prediction interval at $x_{n+1}$
   \STATE Fit a model $\mu_{\hat z}$ on the training data $\Data_{n+1}(\hat z)$
   \STATE Compute the quantile $Q_{1-\alpha}(\hat z) = U_{(\lceil (1-\alpha)(n+1) \rceil)}(z, \hat z)$ where the $U_i$s are defined in \Cref{prop:bounding_pi}
   \STATE {\bfseries Return:} $[\mu_{\hat z}(x_{n+1}) \pm (Q_{1-\alpha}(\hat z) + \tau_{n+1})]$
\end{algorithmic}
\end{algorithm}

\subsection{Batch Approximation}
\label{subsec:Batch_Approximation}

The stable conformal sets require a single model fit $\mu_{\hat z}$ for an arbitrary candidate $\hat z$. The approximation gaps are computable as
$$ \max\{ \pi(z) - \pi_{\rm{lo}}(z, \hat z), \pi_{\rm{up}}(z, \hat z) - \pi(z) \} \leq \rm{Gap}(z, \hat z) \enspace,$$
where
$$ \rm{Gap}(z, \hat z) := \pi_{\rm{up}}(z, \hat z) - \pi_{\rm{lo}}(z, \hat z) \enspace.
$$
Since the above upper and lower bounds hold for any $\hat z$, tighter approximations are obtained with a batch of candidates $\mathcal{Z} = \hat z_1, \cdots, \hat z_d$ as
$$\pi_{\rm{up}}(z, \mathcal{Z}) = \inf_{\hat z \in \mathcal{Z}} \pi_{\rm{up}}(z, \hat z) \text{ and }
 \pi_{\rm{lo}}(z, \mathcal{Z}) = \sup_{\hat z \in \mathcal{Z}} \pi_{\rm{lo}}(z, \hat z) \enspace.
$$

Another possibility is to build an interpolation of $z \mapsto \mu_z(\cdot)$ based on query points $\hat z_1, \cdots, \hat z_d \in (z_{\min}, z_{\max}) \subset \bbR$. For example, one can consider as predictive model the following piecewise linear interpolation
\begin{equation*}
% \label{eq:piece_linear_interpolation}
\tilde\mu_{z} =
\begin{cases}
\frac{\hat z_1 - z}{\hat z_1 - z_{\min}} \mu_{z_{\min}} + \frac{z_{\min} - z}{\hat z_1 - z_{\min}} \mu_{\hat z_1} &\text{ if } z \leq z_{\min} \enspace, \\
\frac{z - \hat z_{t+1}}{\hat z_t - \hat z_{t+1}} \mu_{\hat z_t} + \frac{z - \hat z_{t}}{\hat z_{t+1} - \hat z_t} \mu_{\hat z_{t+1}} & \text{ if } z \in [\hat z_t, \hat z_{t+1}]\enspace, \\
\frac{z - \hat z_d}{z_{\max} - \hat z_d} \mu_{z_{\max}} + \frac{z_{\max} - z}{z_{\max} - \hat z_d} \mu_{\hat z_d} &\text{ if } z \geq z_{\max}\enspace,
\end{cases}
\end{equation*}
% 
% Then $\tilde \mu_{\cdot}$ preserves the stability while providing a potentially better approximation. Indeed, using the triangle inequality, we have
% \todo{move the technical details in the appendix}
% % 
% \begin{align*}
% \tilde{\mathrm{stab}} &:= |S(q, \tilde \mu_{z}(x_i)) - S(q, \tilde \mu_{z_0}(x_i))| \\
% % 
% &\leq |S(q, \tilde \mu_{z}(x_i)) - S(q, \mu_{z}(x_i))| \\
% &\qquad + |S(q, \mu_{z}(x_i)) - S(q, \mu_{z_0}(x_i))| \\
% &\qquad + |S(q, \mu_{z_0}(x_i)) - S(q, \tilde \mu_{z_0}(x_i))|
% \end{align*}

% If $\mu_{\cdot}$ is stable, then the second term of the right hand side of the previous inequality is bounded by $\tau_i$. Now, assuming that $S$ is sufficiently regular \eg $1$-Lipschitz in its second argument, then we have:
% % 
% $$|S(q, \tilde \mu_{z}(x_i)) - S(q, \mu_{z}(x_i))| \leq \mathcal{E}_{z}^{i}\enspace,$$
% % 
% where
% % 
% \begin{align*}
% \mathcal{E}_{z}^{i} &= |\mu_z(x_i) - \tilde \mu_z(x_i)|\\
%  &\leq |\mu_z(x_i) - \alpha_t \mu_{z_t}(x_i) - (1 - \alpha_t)\mu_{z_{t+1}}(x_i)| \\
% % 
% &\leq \alpha_t|\mu_z(x_i) - \mu_{z_t}(x_i)| + (1 - \alpha_t)|\mu_z(x_i) - \mu_{z_{t+1}}(x_i)| \\
% &\leq \alpha_t \tau_i + (1 - \alpha_t) \tau_i = \tau_i \enspace,
% \end{align*}
% % 
% with $\alpha_t \in \left\{\frac{z_1 - z}{z_1 - z_{\min}}, \frac{z - z_{t+1}}{z_t - z_{t+1}}, \frac{z - z_d}{z_{\max} - z_d}\right\}$ is the scaling of interpolation points. Thus, we obtain
% % 
% $$\tilde{\mathrm{stab}} \leq  \mathcal{E}_{z}^{i} + \tau_i +  \mathcal{E}_{z_{0}}^{i} \leq 3\tau_i \enspace.$$
% 
An important point is that, by using the stability bound, the coverage guarantee of the interpolated conformal set is preserved without the need of the expensive symmetrization proposed in \citep{Ndiaye_Takeuchi21}. Such techniques are more relevant when the sample size is small or when  precise estimates of the stability bounds are not available. The corresponding conformity function is defined in a similar way as the previous versions, where we simply plugin the interpolated model. We refer to the appendix for more details. \looseness=-1

% The upper and lower approximation of the conformity function obtained with the interpolated model fit along with stability bounds are defined as:
% % 
% \begin{align*}
% \tilde \pi_{\rm{lo}}(z) &= 1 - \frac{1}{n+1} \sum_{i=1}^{n+1} \mathbb{1}_{\tilde L_i(z) \leq \tilde U_{n+1}(z)} \enspace,\\
% % 
% \tilde \pi_{\rm{up}}(z) &= 1 - \frac{1}{n+1} \sum_{i=1}^{n+1} \mathbb{1}_{\tilde U_i(z) \leq \tilde L_{n+1}(z)} \enspace,
% \end{align*}
% % 
% where for any index $i$ in $[n + 1]$, 
% \begin{align*}
% &\tilde L_i(z) = \tilde E_i(z) - \tau_i \text{ and } \tilde U_i(z) = \tilde E_i(z) + \tau_i \\
% &\tilde E_i(z) =  S(y_i, \tilde \mu_{z}(x_{i})) \text{ and } \tilde E_{n+1}(z) =  S(z, \tilde \mu_{z}(x_{i})) \enspace.
% \end{align*}

\begin{remark}[Categorical Variables]
  In this article, we have essentially limited ourselves to regression problems which, in general, pose intractable computational difficulties. However, the methods remain applicable for classification problems where the set of candidates can only take a finite number of values in $ \mathcal{C} := {c_1, \ldots, c_m}$. In this case, an additional precaution of encoding the categories in real numbers is necessary. Considering the leave-one-out score function, our proposal is therefore an alternative to the approximations via influence function used in \citep{Alaa_VanDerSchaar20, Abad_Bhatt_Weller_Cherubin22} when an exact computation \citep{Cherubin_Chatzikokolakis_Jaggi21} would be unusable or too costly.
\end{remark}

%% file: subfiles/stability_bounds.tex
\subsection{Stability Bounds}
\label{subsec:Stability_bounds}

In this section, we recall some stability bounds. The proof techniques rely on regularity assumptions on the function to be minimized and are relatively standard in optimization \citep[Chapter 13]{ShalevShwartz_BenDavid14}. Stability is a widely used assumption to provide generalization bounds for machine learning algorithms \citep{Bousquet_Elisseeff02, Hardt_Recht_Singer16}. We specify that here the notion of stability that we require is related to the variation of the score and not of the loss function in the optimization objective. However, the ideas for establishing the stability bounds are essentially the same and we recall the core strategies here for the sake of completeness.\\

Let us start with the unregularized model where $\Omega = 0$ \ie 
% the prediction function be defined as $\mu_z(x) = \Phi(x, \beta(z))$ where
% 
\begin{equation}\label{eq:unregularized_min}
\beta(z) \in \argmin_{\beta \in \bbR^p} \mathcal{L}(y(z), \Phi(X, \beta)) = F_z(\Phi(X, \beta)) \enspace.
\end{equation}

\begin{definition}\label{def:sc}
A function $f$ is $\lambda$-strongly convex if for any $w_0, w$ and $\varsigma \in (0,1)$
\begin{align*}
f(\varsigma w_0 + (1 - \varsigma) w) &\leq \varsigma f(w_0) + (1 - \varsigma)f(w) \\
&\qquad- \frac{\lambda}{2} \varsigma (1-\varsigma) \norm{w_0 - w}^2 \enspace.
\end{align*}
\end{definition}

\begin{proposition}\label{prop:stability_sc_loss}
Assume that for any $z$, $F_z$ is $\lambda$-strongly convex and $\rho$-Lipschitz. It holds
$$\norm{\mu_{z}(X) - \mu_{z_0}(X)} \leq \frac{2 \rho}{\lambda} \enspace.$$
\end{proposition}

\begin{proof}
By optimality of $\beta(z)$, we have
\begin{equation}\label{eq:min}
F_z(\Phi(X, \beta(z))) \leq F_z(\Phi(X, \beta)) \quad \forall \beta \enspace.
\end{equation}
We simply apply the optimality condition and strong convexity of the function $F_z$ to the vectors $w_0=\Phi(X, \beta(z_0))=\mu_{z_0}(X)$ and $w = \Phi(X, \beta(z))=\mu_{z}(X)$, it holds
\begin{align*}
0 &\overset{\eqref{eq:min}}{\leq} \frac{F_z(\varsigma w_0 + (1 - \varsigma) w) - F_z(w)}{\varsigma}\\
&\overset{\eqref{def:sc}}{\leq} F_z(w_0) - F_z(w) - \frac{\lambda}{2}(1 - \varsigma)\norm{w_0 - w}^2 \enspace.
\end{align*}
Since $F_z$ is $\rho$-Lipschitz, we have
$$\frac{\lambda}{2}\norm{w_0 - w}^2 \leq F_z(w_0) - F_z(w) \leq \rho \norm{w - w_0} \enspace.$$
% If we don't assume Lipschitzness, a trivial bound is
% $$\frac{\lambda}{2}\norm{\mu_{z}(X) - \mu_{z_0}(X)}^2 \leq F_z(\mu_{z_0}(X)) - F_z(\mu_{z}(X)) \leq F_z(\mu_{z_0}(X)) \leq F_z(0)$$
% We remind that, the vector
% $w_0=\Phi(X, \beta(z_0)) = \mu_{z_0}(X)$ and $w = \Phi(X, \beta(z)) = \mu_z(X)$. 
Therefore, $\frac{\lambda}{2}\norm{w_0 - w} \leq \rho$, hence the result.
% $$\frac{\lambda}{2}\norm{\mu_{z}(X) - \mu_{z_0}(X)} \leq \rho \enspace.$$
\end{proof}

The \Cref{prop:stability_sc_loss} does not assume that the optimization problem in \Cref{eq:unregularized_min} is convex in the model parameter $\beta$. We can now easily deduce a stability bound according to the \Cref{def:Algorithmic_Stability}.

\begin{corollary}
If the score function $S(q, \cdot)$ is $\gamma$-Lipschitz for any $q$, then
$$\tau_i = \frac{2 \gamma \rho}{\lambda}, \qquad \forall i \in [n+1] \enspace.$$
\end{corollary}

When the loss function is not strongly convex, it is known that adding a strongly convex regularization can stabilize the algorithm \citep[Chapter 13]{ShalevShwartz_BenDavid14}.
The proof technique is similar to the previous one with the difference that now the bound is on the $\argmin$ of the optimization problem and not the predictions of the model. This requires stronger assumptions.

\begin{proposition}Assume the optimization problem \Cref{eq:model_optimization} is convex, $\Omega$ is $\lambda$-strongly convex.
If the loss $\mathcal{L}$ is convex-$\rho$-Lipschitz, then
$$ \norm{\beta(z) - \beta(z_0)} \leq \frac{2 \rho}{\lambda} \enspace. $$
When the loss function $\mathcal{L}$ is convex-$\nu$-smooth with $\nu < \lambda$ and $\mathcal{L}(y(z), \mu_z(X)) \leq C$ for any $z$, then
$$ \norm{\beta(z) - \beta(z_0)} \leq \frac{ 2\sqrt{2 \nu C}}{\lambda - \nu} \enspace.$$
\end{proposition}

These optimization error bounds also imply the following stability bounds.
\begin{corollary}
Assume that the score function $S(q, \cdot)$ is $\gamma$-Lipschitz for any $q$, and that the prediction model $\mu_{\cdot}(x) := \Phi(x, \beta(\cdot))$ satisfies for any $x \in \bbR^p,\, z, z_0 \in \bbR$,
$$ |\mu_z(x) - \mu_{z_0}(x)| \leq L_{\Phi}|x^{\top} \beta(z) - x^\top \beta(z_0)| \enspace.$$
If the loss is $\rho$-Lipschitz, then $$\tau_i = \frac{2 \gamma \rho L_{\Phi}\norm{x_i}}{\lambda}, \qquad \forall i \in [n+1]\enspace.$$
If the loss is $\nu$-smooth with $\nu < \lambda$ and bounded by C, then $$\tau_i = \frac{2\gamma L_{\Phi} \norm{x_i} \sqrt{2\nu C}}{\lambda - \nu}, \qquad \forall i \in [n+1] \enspace.$$
\end{corollary}

Another way to understand such regularized bounds, is to leverage duality. A smoothness assumption in the primal space will translate into a strongly concave assumption in the dual space \citep[Theorem 4.2.2, p. 83]{Hiriart-Urruty_Lemarechal93b}. The dual formulation \citep[Chapter 31]{Rockafellar97} of \Cref{eq:model_optimization} reads:
\begin{align}
\theta(z) &\in \argmax_{ \theta \in \bbR^{n+1}} -\mathcal{L}^{*}(y(z), -\theta) - \Omega^*(X^\top \theta) \label{eq:dual_model}
\enspace,
\end{align}
where, given a proper, closed and convex function $f: \bbR^n \to \bbR \cup \{+\infty\}$, we denoted its Fenchel-Legendre transform as $f^*:\bbR^n \to \bbR \cup \{+\infty\}$ defined by $f^*(x^*) = \sup_{x \in \dom f} \langle x^* , x \rangle - f(x)$ with $\dom f = \{x \in \bbR^n: f(x) < +\infty\}$.\\

Let $P_z$ and $D_z$ denote the primal and dual objective functions.
We have the following classical error bounds for the dual optimization problem.
If the loss function $\mathcal{L}$ is $\nu$-smooth, then $\mathcal{L}^*$ is $1/\nu$-strongly convex and we have for $\forall (\beta, \theta) \in \dom P_z \times \dom D_z$
\begin{align*}
\frac{1}{ 2\nu}\norm{\theta(z) - \theta}^2 &\leq D_z(\theta(z)) - D_z(\theta)\\
&= P_z(\beta(z)) - D_z(\theta) \\
&\leq \mathrm{ Duality\_Gap}_z(\beta, \theta) \enspace,
\end{align*}
where the equality follows from strong duality and we recall from weak duality 
% ($D_z(\theta(z)) \leq P_z(\beta), \forall \beta \in \dom P_z$)
 that the duality gap upper bounds the optimization error as follow:
\begin{align*}
\mathrm{Duality\_Gap}_z(\beta, \theta) &:= P_z(\beta) - D_z(\theta)\\
&\geq P_z(\beta) - P_z(\beta(z)) \enspace.
\end{align*}

This readily leads to several possible bounds. If the dual function $D_z(\cdot)$ is $\rho^*$-Lipschitz for any $z$, then $$\norm{\theta(z) - \theta} \leq \sqrt{2\nu \rho^*} \enspace.$$ If the duality gap can be assumed to be bounded by $C$ for any $z \in [z_{\min}, z_{\max}]$, then $$\norm{\theta(z) - \theta} \leq \sqrt{2\nu C} \enspace.$$ We obtain stability bounds when one uses  the dual solution (which is a function of the residual) as a conformity score $S(y(z), \mu_z(X)) = |\theta(z)|$ where the absolute value is taken coordinate wise. For example, these dual based score functions were used in \citep{Ndiaye_Takeuchi19}.

\begin{remark}[Bound on the loss]
The assumption of a bounded loss function that we make, is not rigorously feasible and some adaptations are necessary. For simplicity, let us consider that $\Phi(x, 0)=0$ and $\Omega(0)=0$. Using the optimality of $\beta(z)$, we obtain for any candidate $z$
\begin{align*}
\mathcal{L}(y(z), \mu_z(X)) &\leq \mathcal{L}(y(z), \mu_z(X)) + \Omega(\beta(z)) \\
&\leq \mathcal{L}(y(z), 0) \enspace.
\end{align*}
Unfortunately, for common examples such as least squares, the right hand side is unbounded. Nevertheless, since the data are assumed to be exchangeable, we have
$$\mathbb{P}(y_{n+1} \in [y_{(1)}, y_{(n)}]) \geq 1 - \frac{2}{n+1} \enspace.$$
Hence it is reasonable to restrict the range of candidates as $z \in [y_{(1)}, y_{(n)}]$, which implies
$$\mathcal{L}(y(z), \mu_z(X)) \leq \sup_{z \in [y_{(1)}, y_{(n)}]} \mathcal{L}(y(z), 0) =: C \enspace.$$
\end{remark}

%% file: subfiles/experiments.tex
\section{Numerical Experiments}
% \begin{itemize}
% \item The leave-one out conformal set loose a factor in the coverage guarantee. Maybe exploiting algorithmic stability can help!
% \item Need to try it too when exchangeability is not there.
% \end{itemize}

\begin{figure*}
\centering
\subfigure[\texttt{Boston} $(506, 13)$]{\includegraphics[width=0.49\linewidth]{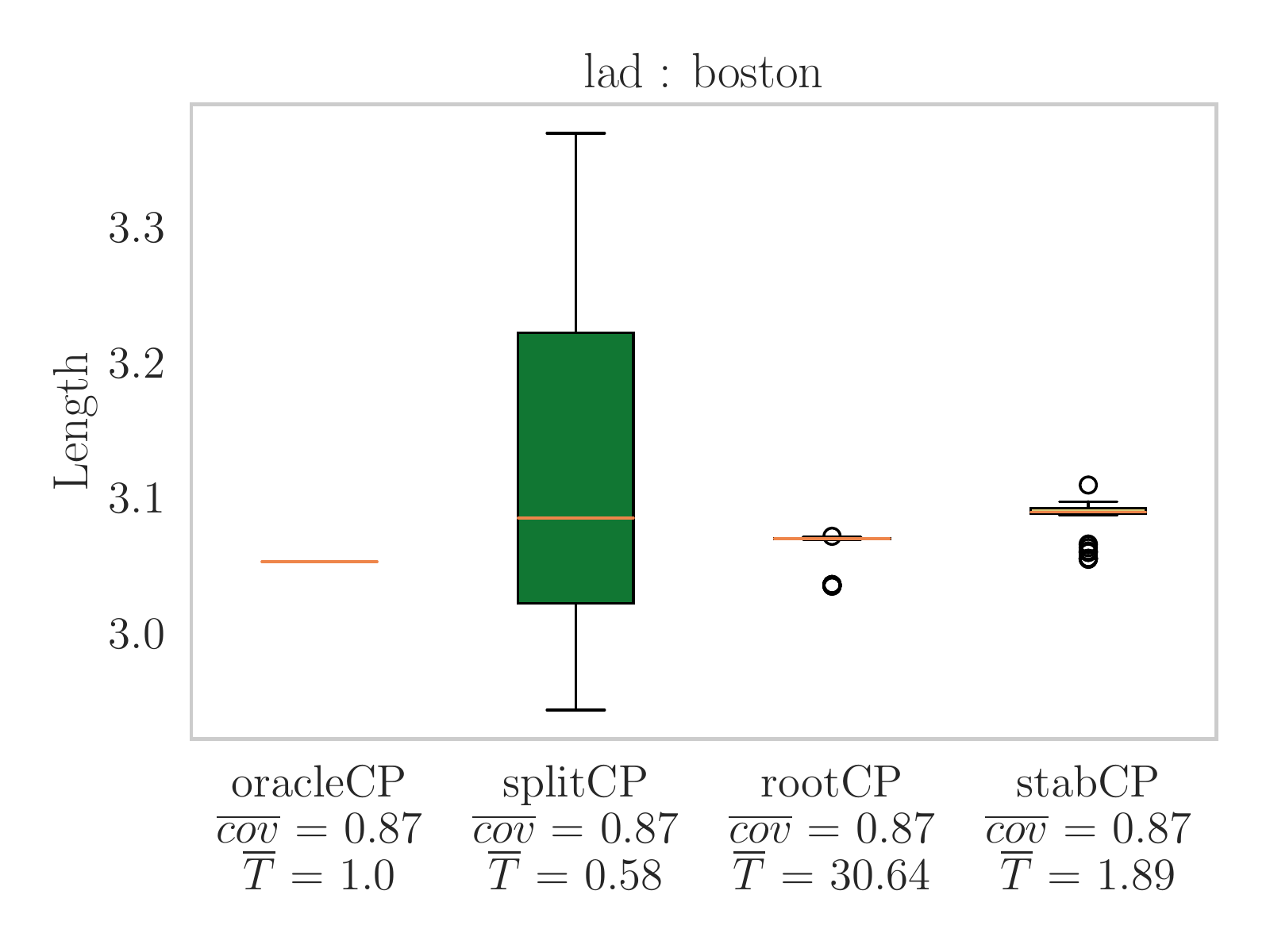}}
\subfigure[\texttt{Diabetes} $(442, 10)$]{\includegraphics[width=0.49\linewidth]{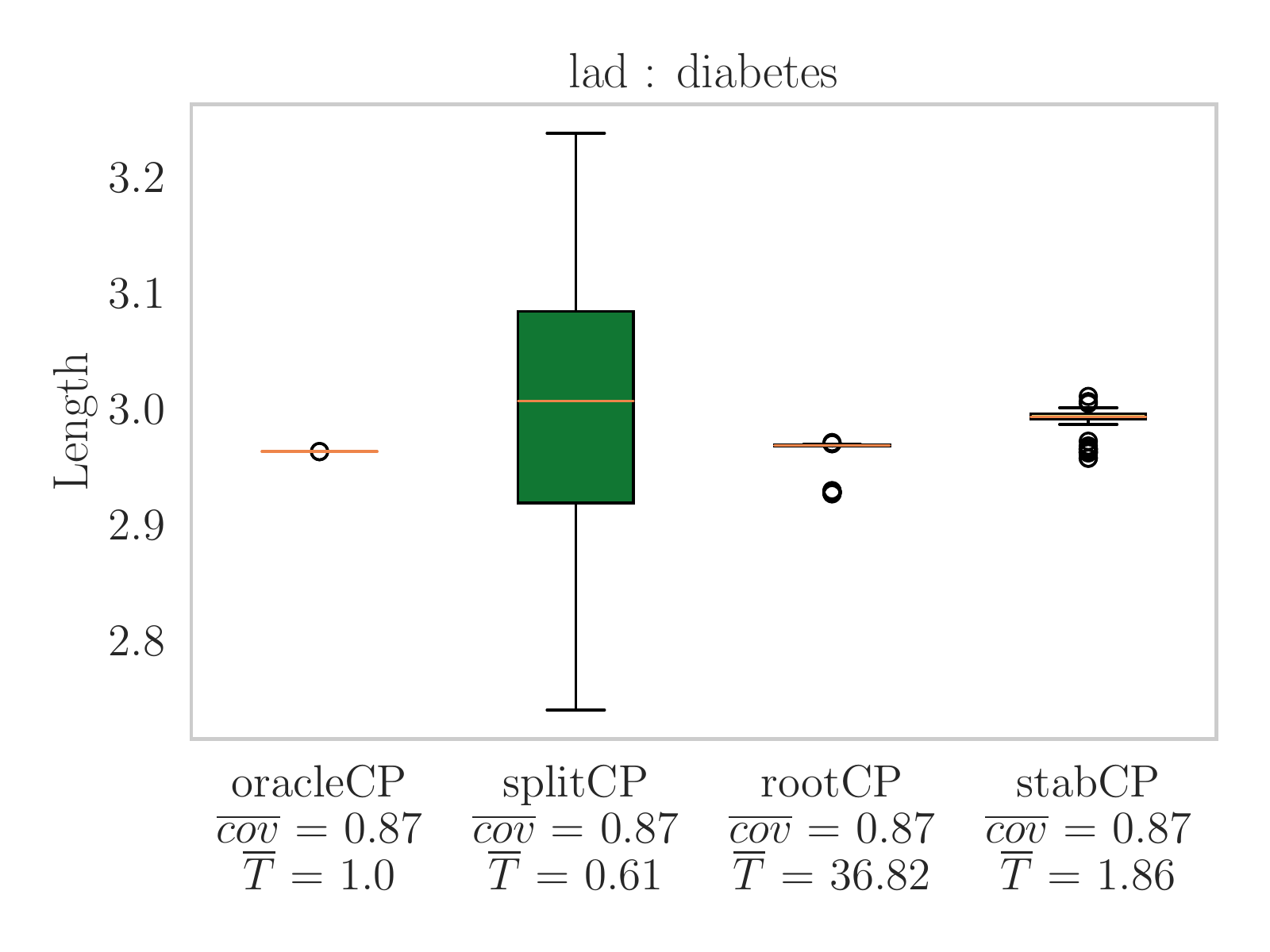}}
\subfigure[\texttt{Housingcalifornia} $(20640, 8)$]{\includegraphics[width=0.49\linewidth]{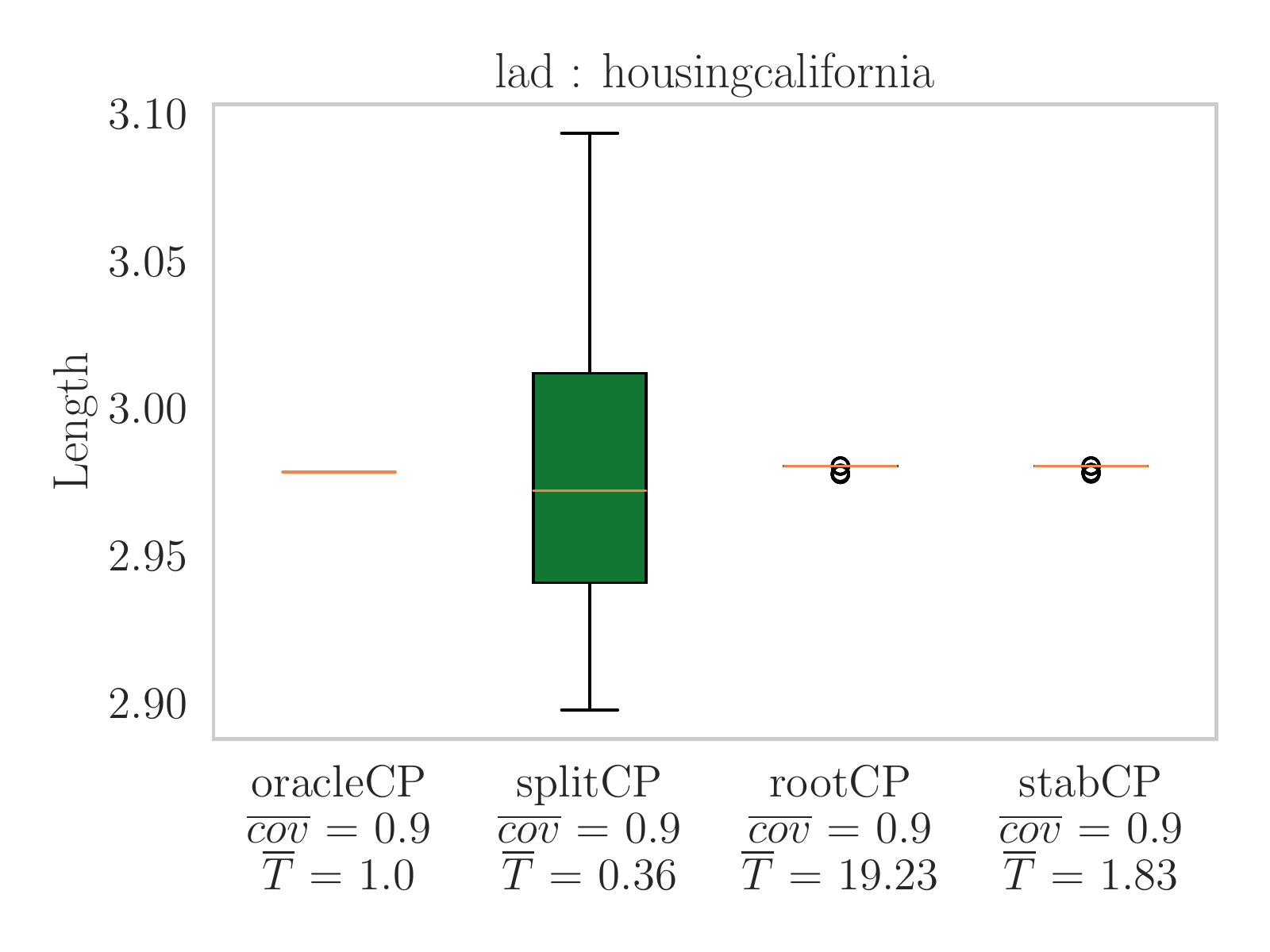}}
\subfigure[\texttt{Friedman1} $(500, 100)$]{\includegraphics[width=0.49\linewidth]{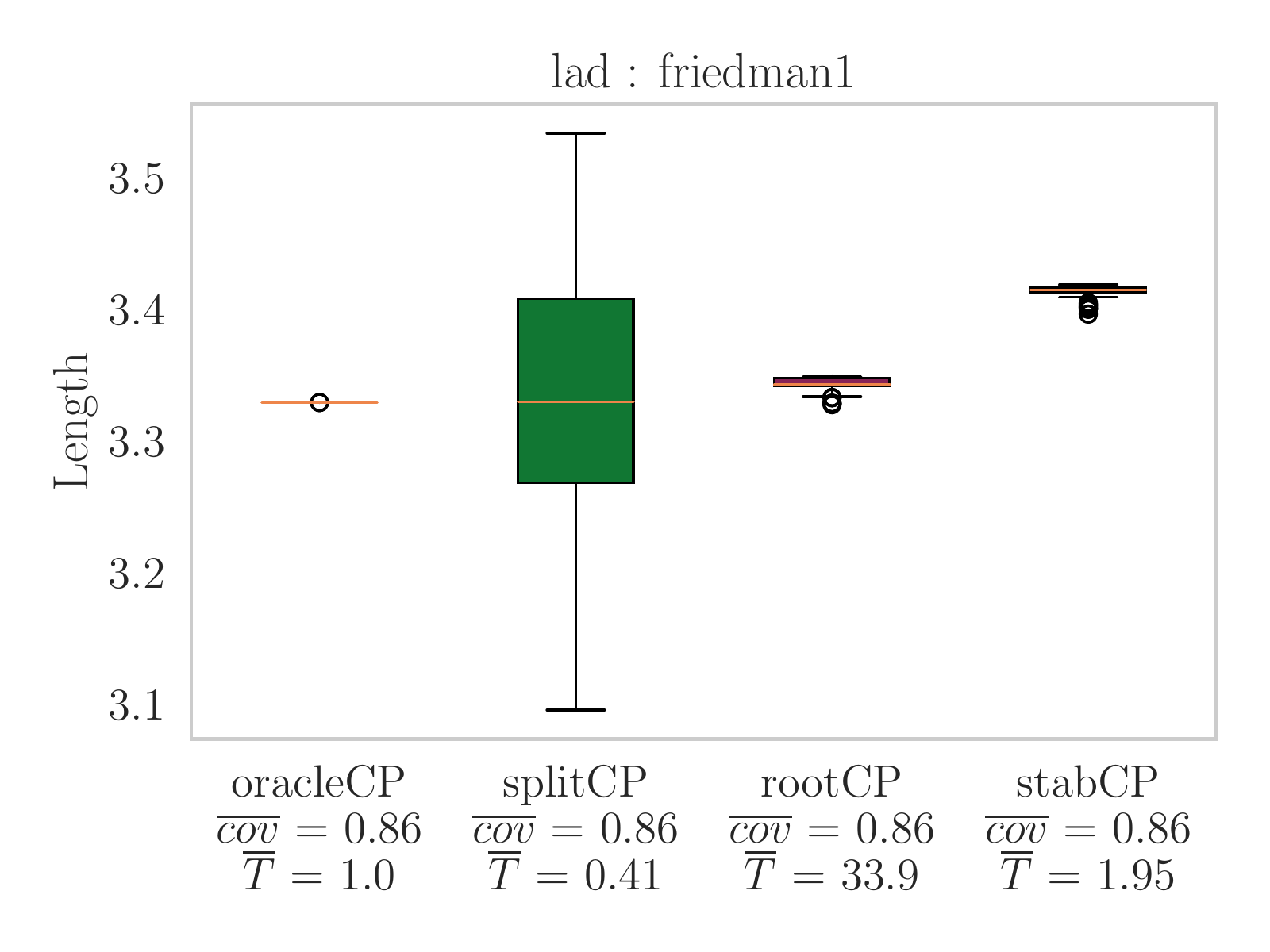}}
\caption{Benchmarking conformal sets for the least absolute deviation regression models with a ridge regularization on real datasets. We display the lengths of the confidence sets over $100$ random permutation of the data. We denoted $\overline{cov}$ the average coverage and $\overline{T}$ the average computational time normalized with the average time for computing \texttt{oracleCP} which requires a single full data model fit. The full and exact CP set can always be approximated with a fine (costly) grid discretization of the output space and can then be used as a default baseline. Here, it is represented by rootCP since in the examples displayed the full CP set turns out to be an interval and then rootCP is equal to the full CP up to $\epsilon_r$ digit precision on the decimals; we used a default value of $\epsilon_r=10^{-4}$.\label{fig:benchmarks}}
\end{figure*}

We conduct all the experiments with a coverage level of $0.9$ \ie $\alpha = 0.1$. For comparisons, we run the evaluations on $100$ repetitions of examples and display the average of the following performance statistics for different methods: the empirical coverage \ie the percentage of times the prediction set contains the held-out target $y_{n+1}$, the length of the confidence intervals, and the execution time. We compare the method we propose \texttt{stabCP} with the conformal prediction set computed with an oracle method defined below, with a splitting strategy \texttt{splitCP} \citep{Papadopoulos_Proedrou_Vovk_Gammerman02, Lei_GSell_Rinaldo_Tibshirani_Wasserman18}, and finally with an estimation of the $\alpha$-level set of the conformity function \texttt{rootCP} \citep{Ndiaye_Takeuchi21} by root-finding solvers. Note that, when the conformal set is a bounded interval, \texttt{stabCP} approximates \texttt{rootCP} as in \Cref{fig:illustration_stab}. In all experiments conducted, we observed that the exact conformal prediction set is indeed an interval. Although this is often the case, we recall that it might not be in general. Just for the comparisons, we therefore estimated the \texttt{stabCP} sets with a root-finding solver as well, as if a closed form solution was not available. A \texttt{python} package with our implementation is available at 
\url{https://github.com/EugeneNdiaye/stable_conformal_prediction}
where additional numerical experiments (\eg using large pre-trained neural net) and benchmarks will be provided. \looseness=-1

\paragraph{\texttt{oracleCP}.}
To define an oracle prediction set as reference, we follow in \citep{Ndiaye_Takeuchi19, Ndiaye_Takeuchi21} and assume that the unavailable target variable $y_{n+1}$ is observed by the algorithm. Hence, we define the oracle scores
\begin{align*}
\forall i \in [n], \quad E_{i}^{\rm{or}} &= S(y_i,\, \mu_{y_{n+1}}(x_i)) \enspace,\\
E_{n+1}^{\rm{or}}(z) &= S(z,\, \mu_{y_{n+1}}(x_{n+1})) \enspace,
\end{align*}
and the oracle conformal set as
\begin{align*}
\Gamma_{\rm{oracle}}^{(\alpha)}(x_{n+1}) &:= \{z: \pi_{\rm{oracle}}(z) \geq \alpha\} \enspace,\\
\pi_{\rm{oracle}}(z) &= 1 - \frac{1}{n + 1}\sum_{i=1}^{n+1} \mathbb{1}_{E_{i}^{\rm{or}} \leq E_{n+1}^{\rm{or}}(z)} \enspace.
\end{align*}

\paragraph{\texttt{splitCP}.} 
A popular and classical estimation of conformal prediction sets relies on splitting the dataset. The split conformal prediction set introduced in \citep{Papadopoulos_Proedrou_Vovk_Gammerman02}, separates the model fitting and the calibration steps. Let us define
\begin{itemize}
\item the training set 
$$\mathcal{D}_{\rm{tr}} = \{(x_1, y_1), \cdots, (x_m, y_m)\} \text{ with } m < n \enspace,$$
\item the calibration set 
$$\mathcal{D}_{\rm{cal}} = \{(x_{m+1}, y_{m+1}), \cdots, (x_n, y_n)\} \enspace.$$
\end{itemize}
Then the model is fitted on the training set $\mathcal{D}_{\rm{tr}}$ to get $\mu_{\rm{tr}}(\cdot)$ and define the score function on the calibration set $\mathcal{D}_{\rm{cal}}$:
\begin{align*}
\forall i \in [m+1, n], \quad  E_{i}^{\rm{cal}} &= S(y_i,\, \mu_{\rm{tr}}(x_i)) \enspace,\\
E_{n+1}^{\rm{cal}}(z) &= S(z,\, \mu_{\rm{tr}}(x_{n+1})) \enspace.
\end{align*}
Thus, we obtain the split conformal set as
\begin{align*}
\Gamma_{\rm{split}}^{(\alpha)}(x_{n+1}) &= \{z: \pi_{\rm{split}}(z) \geq \alpha \} \enspace,\\
\pi_{\rm{split}}(z) &= 1 - \frac{1}{n - m + 1}\sum_{i=m+1}^{n+1} \mathbb{1}_{E_{i}^{\rm{cal}} \leq E_{n+1}^{\rm{cal}}(z)} \enspace.
\end{align*}

\section{Discussion}
The data splitting approach does not use all the data in the training phase. It is often less statistically efficient, and its interval length can vary greatly depending on the additional randomness of the split. On the contrary, our approach does not use any splitting, provides an approximation of the exact conformal set that is pretty accurate depending on the stability of the model as can be observed on \Cref{fig:benchmarks}. All this requires one and only one data fitting of the underlying learning model. You will notice that \texttt{splitCP} and \texttt{stabCP} have the same structure and are simple intervals if the score functions are reasonably simple. The presence of data splitting in the former is replaced by an additional stability term in the latter. So if the predictive model is very stable, \texttt{stabCP} benefits from all the data, and very little regularization to get closer to the oracle version that includes the unknown target $y_{n+1}$. To date, we are not aware of any other method that can obtain a full conformal prediction set with such computational efficiency while ensuring no loss on the coverage guarantee.
We observe on the benchmarks with real data \Cref{fig:benchmarks} that the \texttt{stabCP} is often very similar to the \texttt{rootCP} which approximates with a very fine precision the exact set (under the assumption that the latter is a bounded interval). Our proposal has the net advantage of being twenty to thirty times faster and can often be computed in closed form.\looseness=-1

However, as can be seen in \Cref{fig:illustration_stab}, our proposed method loses precision when the sample size is small. This reflects the difficulty of estimating a reliable confidence set in the absence of algorithmic stability. At the same time, it is difficult to have an algorithm that generalizes well with so little training data. Otherwise, when the size of the data is important, the influence of the stability bound is very little felt because they are often of the order of magnitude $O(1/n)$.\looseness=-1

Finally, a notorious limitation is that one needs to know explicitly the stability bounds. This can be difficult to estimate for some models. The bounds we presented in \Cref{subsec:Stability_bounds} cover a wide range of examples and can be completed by bounds displayed in \cite{Hardt_Recht_Singer16, Bassily_Feldman_Guzman_Talwar20, Lei_Yang_Yang_Ying21, Klochkov_Zhivotovskiy21} for stochastic gradient descent. Even if the notion of stability required here is slightly different, any error bound on the estimator can be naturally converted into a stability bound for conformal prediction sets. So we don't lose much generality as long as we make the assumption that the score function is sufficiently regular \eg Lipschitz. This is precisely what allowed us to obtain the bounds presented in this article. Yet, if the parameter of the predictive model is defined iteratively by a gradient descent process on a non-convex objective function, obtaining stability bounds becomes quite delicate. Moreover, the Lipschitz constant of neural network objectives can be poorly estimated. In this case, our approach could not be applied safely or could lead to uninformative confidence intervals.
The splitting strategy remains more flexible. It would be interesting to study fine combinations of data splitting and inclusion of stability bounds to reduce the size of the confidence intervals and their variance while being pivotal to explicit stability bounds.\looseness=-1

%% file: subfiles/appendix.tex
\section{Appendix}

In these supplementary notes, we complete some proofs and bring algorithmic precisions of our approach as well as additional numerical experiments.

% \subsection{Regularity Constants for Some Popular Loss and Regularization Functions}
% \todo{Complete}

% \paragraph{Strongly-convex regularization}

% \textbf{Ridge} $\frac{\lambda}{2} \norm{\cdot}^2$ and elastic net $\norm{\cdot}_1 + \frac{\lambda}{2} \norm{\cdot}^2$

% \paragraph{Convex-Smooth Loss}

% \paragraph{Convex-Lipschitz Loss}

\subsection{StabCP Set with General Score Function}

We explain a simple procedure to approximate the set prediction with an arbitrary pre-defined accuracy. We recall that
\begin{align*}
    \Gamma_{\rm{up}}^{(\alpha)}(x_{n+1}) &= \{z: \pi_{\rm{up}}(z, \hat z) \geq \alpha \} \\
    &= \{z: S(z, \mu_{\hat z}(x_{n+1})) \leq Q_{1-\alpha}(\hat z) + \tau_{n+1}\} \enspace,
\end{align*}
which is a convex set when the level-set of the score function is convex. By simplicity, we assume that the score function is such that $\Gamma_{\rm{up}}^{(\alpha)}(x_{n+1})$ is a bounded interval. \Cref{alg:stabCP_general} summarizes the process.

\begin{algorithm}[H]
    \caption{Stable conformal prediction set for score function with convex level-set}
    \label{alg:stabCP_general}
 \begin{algorithmic}
    \STATE {\bfseries Input:} data $\{(x_1, y_1), \ldots, (x_n, y_n)\}$ and $x_{n+1}$
    \STATE Coverage level $\alpha \in (0, 1)$, any estimate $\hat z \in \bbR$
    \STATE Stability bounds $\tau_1, \ldots, \tau_{n+1}$ of the learning algorithm
    \STATE {\bfseries Output:} prediction interval at $x_{n+1}$
    \STATE Fit a model $\mu_{\hat z}$ on the training data $\Data_{n+1}(\hat z)$
    \STATE Compute the quantile $Q_{1-\alpha}(\hat z) = U_{(\lceil (1-\alpha)(n+1) \rceil)}(z, \hat z)$ where the $U_i$s are defined in \Cref{prop:bounding_pi}
    \STATE Compute $\Gamma_{\rm{up}}^{(\alpha)}(x_{n+1}) = [\ell_{\alpha}(x_{n+1}), u_{\alpha}(x_{n+1})]$ up to $\epsilon_r > 0$ tolerance error as follow:
    \begin{enumerate}
        \item find $z_{\min} < z_0 < z_{\max}$ such that 
        \begin{equation}\label{eq:initialization_condition}
        \pi_{\rm{up}}(z_{\min}, \hat z) < \alpha < \pi_{\rm{up}}(z_{0}, \hat z) \text{ and } \alpha > \pi_{\rm{up}}(z_{\max}, \hat z) \enspace.
        \end{equation}
        \item Perform a bisection search in $[z_{\min}, z_0]$. It will output a point $\hat \ell$ such that $\ell_{\alpha}(x_{n+1})$ belongs to $[\hat \ell \pm \epsilon_r]$ after at most $\log_2(\frac{z_0 - z_{\min}}{\epsilon_r})$ iterations.
        
        \item Perform a bisection search in $[z_0, z_{\max}]$. It will output a point $\hat u$ such that $u_{\alpha}(x_{n+1})$ belongs to $[\hat u \pm \epsilon_r]$ after at most $\log_2(\frac{z_{\max} - z_0}{\epsilon_r})$ iterations.
        \end{enumerate}
    \STATE {\bfseries Return:} $\Gamma_{\rm{up}}^{(\alpha)}(x_{n+1})$
 \end{algorithmic}
 \end{algorithm}

\subsection{Stability of the Linear Interpolation}

We discussed in \Cref{subsec:Batch_Approximation} the potential gain in accuracy when approximating the conformity function using not a single point but a batch of points. Here we justify the interpolation approach when the score function $S$ is sufficiently regular.

\begin{proposition}
Let us assume that the score function $S(q, \cdot)$ is $\gamma$-Lipschitz for any $q$, and consider the interpolated prediction model defined as 
\begin{equation}
    \label{eq:piece_linear_interpolation}
    \tilde\mu_{z} =
    \begin{cases}
    \frac{\hat z_1 - z}{\hat z_1 - z_{\min}} \mu_{z_{\min}} + \frac{z_{\min} - z}{\hat z_1 - z_{\min}} \mu_{\hat z_1} &\text{ if } z \leq z_{\min} \enspace, \\
    \frac{z - \hat z_{t+1}}{\hat z_t - \hat z_{t+1}} \mu_{\hat z_t} + \frac{z - \hat z_{t}}{\hat z_{t+1} - \hat z_t} \mu_{\hat z_{t+1}} & \text{ if } z \in [\hat z_t, \hat z_{t+1}]\enspace, \\
    \frac{z - \hat z_d}{z_{\max} - \hat z_d} \mu_{z_{\max}} + \frac{z_{\max} - z}{z_{\max} - \hat z_d} \mu_{\hat z_d} &\text{ if } z \geq z_{\max}\enspace,
    \end{cases}
    \end{equation}
where $\mu_{\cdot}$ is stable according to \Cref{def:Algorithmic_Stability}. It holds
\begin{equation} \label{eq:stability_linear_interpolation}
|S(q, \tilde \mu_{z}(x_i)) - S(q, \tilde \mu_{z_0}(x_i))|   \leq 3\gamma\tau_i \enspace.
\end{equation}
\end{proposition}

\begin{proof}
Using the triangle inequality, we have
\begin{align*}
\tilde{\mathrm{stab}} &:= |S(q, \tilde \mu_{z}(x_i)) - S(q, \tilde \mu_{z_0}(x_i))| \\
&\leq |S(q, \tilde \mu_{z}(x_i)) - S(q, \mu_{z}(x_i))| 
 + |S(q, \mu_{z}(x_i)) - S(q, \mu_{z_0}(x_i))| 
 + |S(q, \mu_{z_0}(x_i)) - S(q, \tilde \mu_{z_0}(x_i))| \enspace.
\end{align*}

If $\mu_{\cdot}$ is stable, then the second term of the right hand side of the previous inequality is bounded by $\tau_i$. Now, assuming that $S$ is $\gamma$-Lipschitz in its second argument, for any $q$, we have:
$$|S(q, \tilde \mu_{z}(x_i)) - S(q, \mu_{z}(x_i))| \leq \gamma \mathcal{E}_{z}^{i}\enspace,$$
where
\begin{align*}
\mathcal{E}_{z}^{i} &= |\mu_z(x_i) - \tilde \mu_z(x_i)|\\
 &\leq |\mu_z(x_i) - \alpha_t \mu_{z_t}(x_i) - (1 - \alpha_t)\mu_{z_{t+1}}(x_i)| \\
&\leq \alpha_t|\mu_z(x_i) - \mu_{z_t}(x_i)| + (1 - \alpha_t)|\mu_z(x_i) - \mu_{z_{t+1}}(x_i)| \\
&\leq \alpha_t \tau_i + (1 - \alpha_t) \tau_i = \tau_i \enspace,
\end{align*}
with $\alpha_t \in \left\{\frac{z_1 - z}{z_1 - z_{\min}}, \frac{z - z_{t+1}}{z_t - z_{t+1}}, \frac{z - z_d}{z_{\max} - z_d}\right\}$ is the scaling of interpolation points. Thus, we obtain
$$\tilde{\mathrm{stab}} \leq \gamma (\mathcal{E}_{z}^{i} + \tau_i +  \mathcal{E}_{z_{0}}^{i}) \leq 3\gamma\tau_i \enspace.$$
\end{proof}

The upper and lower approximation of the conformity function obtained with the interpolated model fit along with stability bounds are defined as:
\begin{align*}
\tilde \pi_{\rm{lo}}(z) &= 1 - \frac{1}{n+1} \sum_{i=1}^{n+1} \mathbb{1}_{\tilde L_i(z) \leq \tilde U_{n+1}(z)} \enspace,\\
\tilde \pi_{\rm{up}}(z) &= 1 - \frac{1}{n+1} \sum_{i=1}^{n+1} \mathbb{1}_{\tilde U_i(z) \leq \tilde L_{n+1}(z)} \enspace,
\end{align*}
where for any index $i$ in $[n + 1]$, using the stability bound in \Cref{eq:stability_linear_interpolation}, we define
\begin{align*}
&\tilde L_i(z) = \tilde E_i(z) - 3\gamma\tau_i \text{ and } \tilde U_i(z) = \tilde E_i(z) + 3\gamma\tau_i \enspace, \\
&\tilde E_i(z) =  S(y_i, \tilde \mu_{z}(x_{i})) \text{ and } \tilde E_{n+1}(z) =  S(z, \tilde \mu_{z}(x_{i})) \enspace.
\end{align*}

In general, approximating the entire model path with respect to output/label changes using finite grid points is not always safe for calculating the conformal prediction set because it breaks the exchangeability assumptions of the data set. Incorporating the stability bound will regularize the conformity function to restore the validity of the method. However, the procedure proves to be quite robust to wrong estimation of the stability bounds. The experiments in \cite{Ndiaye_Takeuchi21} are conducted with estimates  $\tau_i=0$ and the prediction sets obtained are essentially the same as the exact one. More detailed experiments will be proposed in our github implementation.

\subsection{Additional Experiments}

In this appendix, we add some numerical experiments to illustrate how \texttt{stabCP} can behave when using an estimator that is not defined as an \texttt{argmin} but rather as an output of an iterative process. In this case, we use a Multi-Layer Perceptron regressor trained with $T=\texttt{n\_iter}$ number of gradient descent iterations. Recent analyses  \cite{Hardt_Recht_Singer16} have shown that any model trained with the stochastic gradient method in a reasonable amount of time achieves low generalization error. The proof of these results consists in showing that the estimator verifies a stability condition when the input data are slightly perturbed. The bounds on the iterates of stochastic gradient methods are often proportional to $\frac{T}{n}$. They also depend on the Lipschitz regularity constants which unfortunately can be hard to estimate in practice. Here, we will be satisfied with the order of magnitude and evaluate the behavior of the conformity function according to the number of iterations performed. We run the experiments on two different datasets with a sample size of $442$ and $20640$.

\begin{figure*}[ht]
\centering
\includegraphics[width=0.49\textwidth]{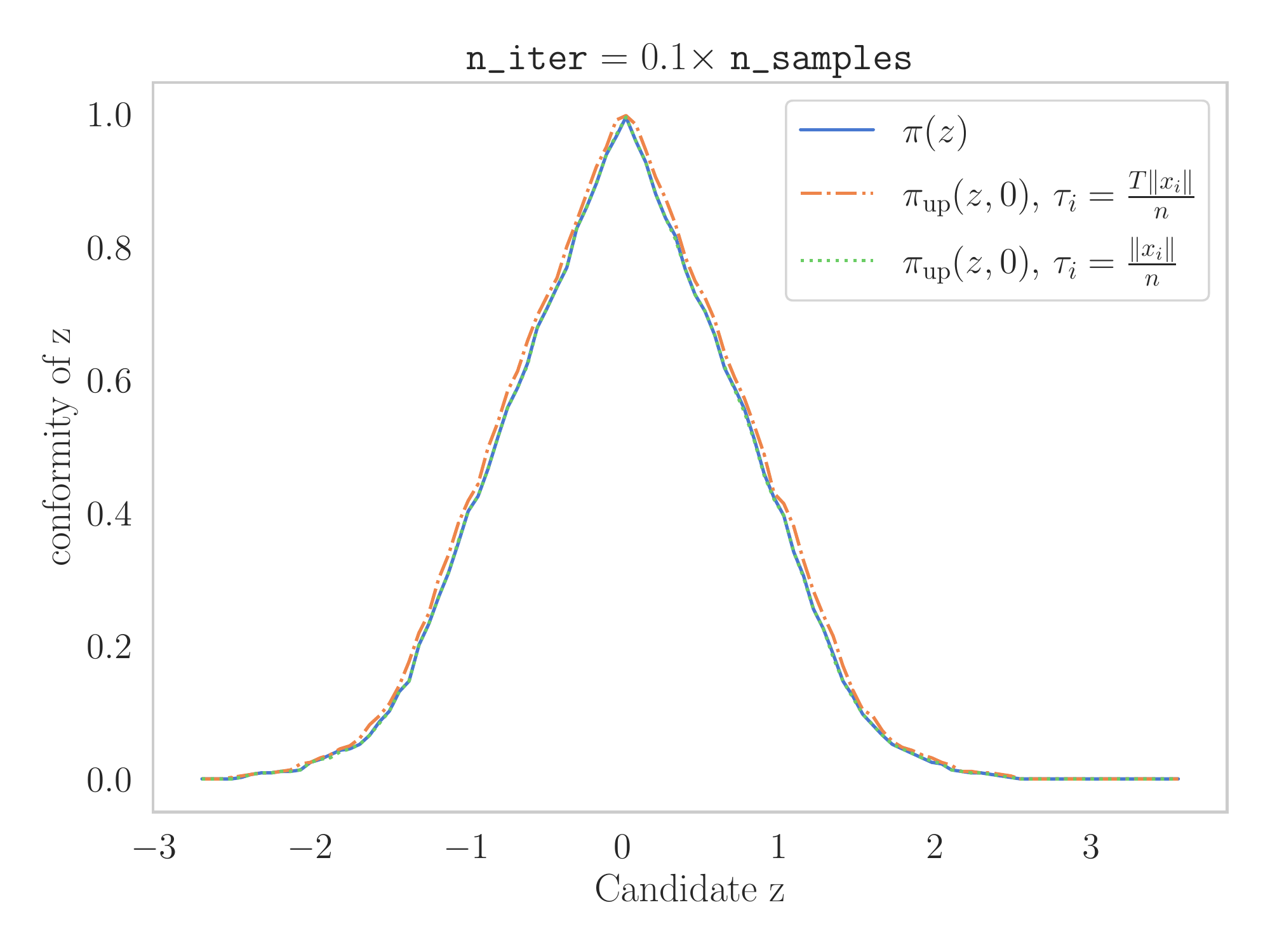}
\includegraphics[width=0.49\textwidth]{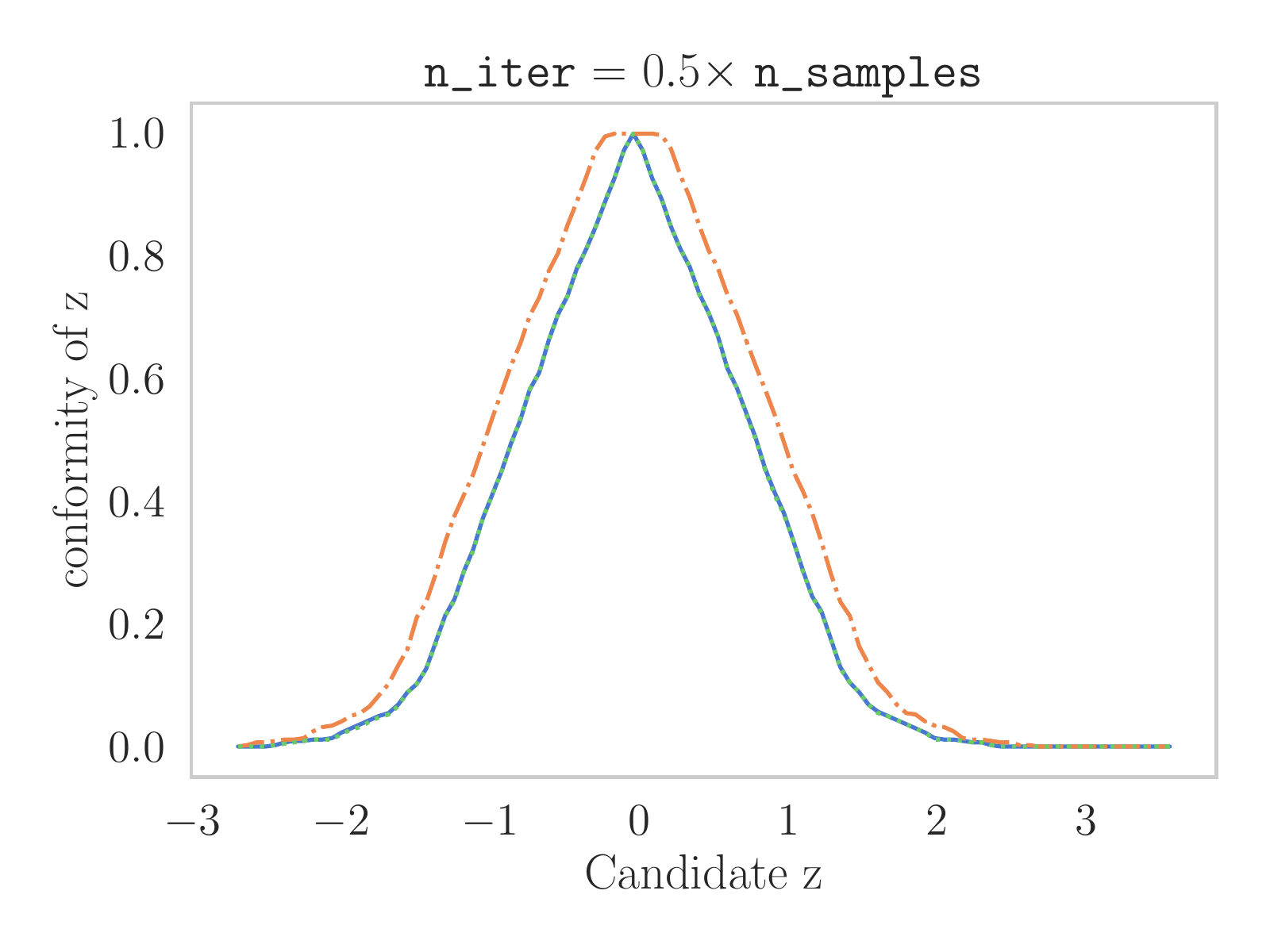}
\includegraphics[width=0.49\textwidth]{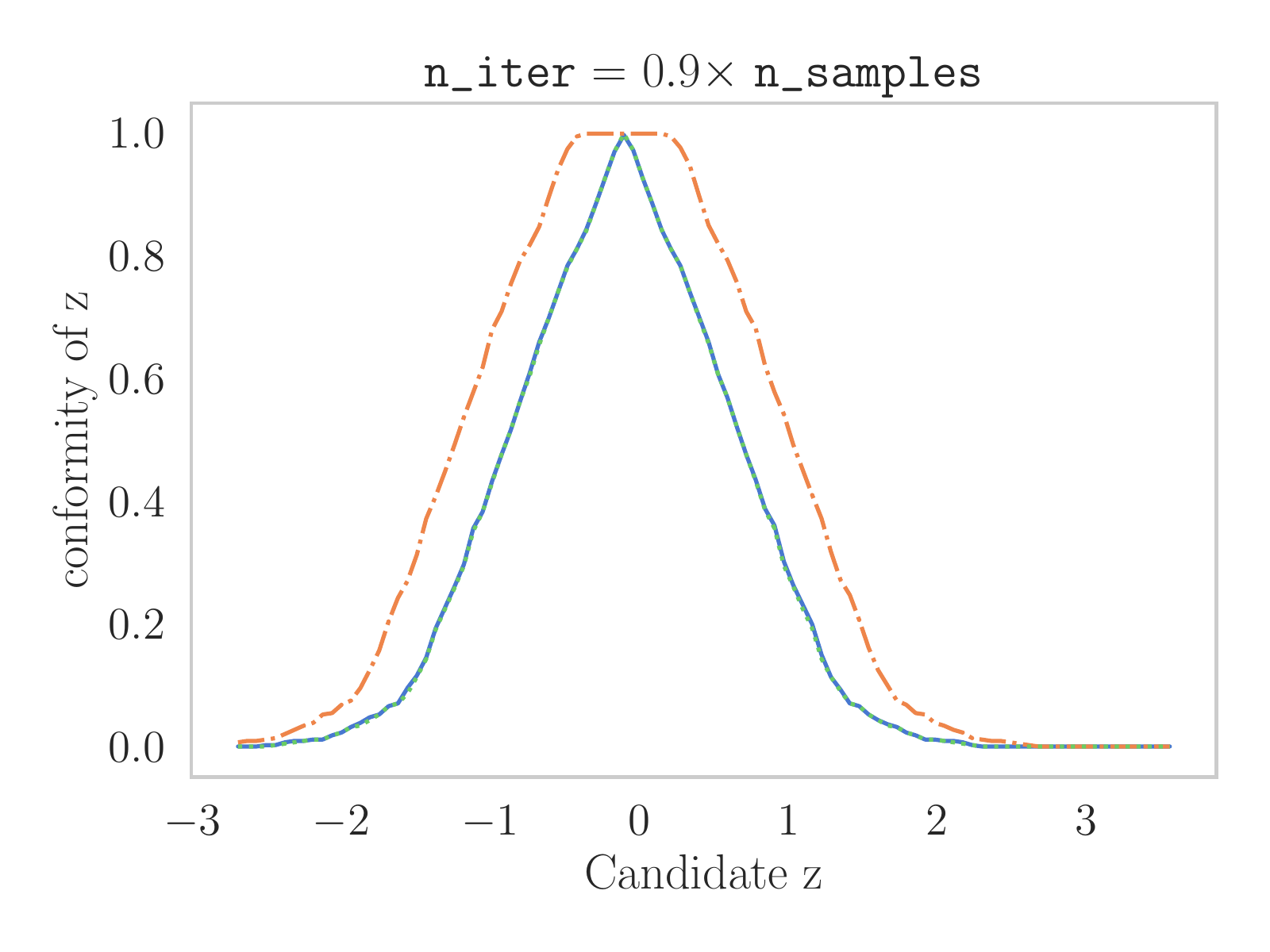}
\includegraphics[width=0.49\textwidth]{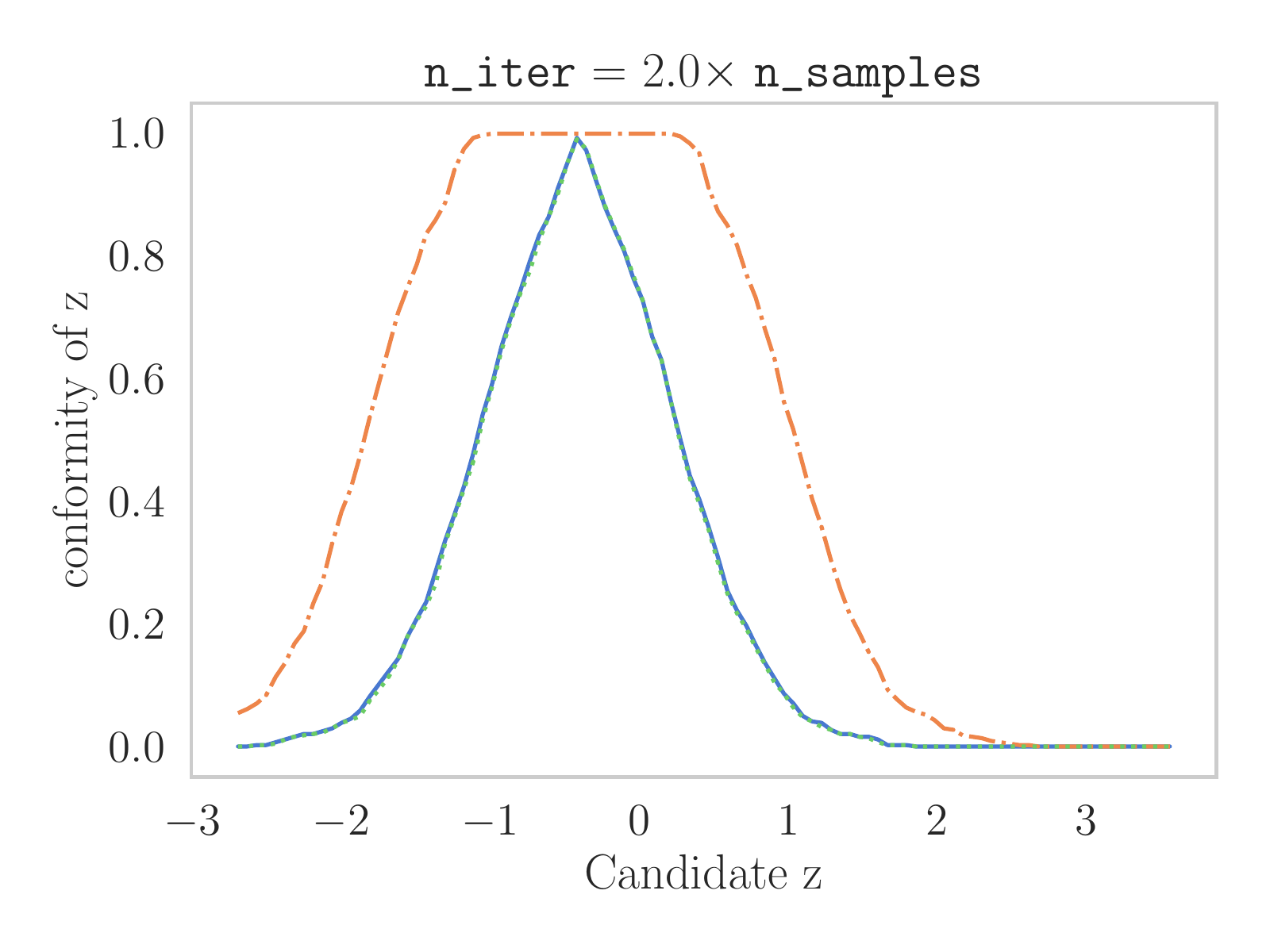}
\caption{Illustration of different conformity functions with respect to a sequence of stability bounds. We observe that by merely staking an order of magnitude $O(1/n)$ as stability bound, gives a good estimate of the conformal prediction set even if the bound is not safe. These experiments are conducted with a Multi-Layer Perceptron regressor on the \texttt{Diabetes} $(442, 10)$ dataset, trained with $T=\texttt{n\_iter}$ iterations of Stochastic Gradient Descent. \label{MLP_diabetes_different_stability}}
\end{figure*}

\begin{figure*}[ht]
\centering
\includegraphics[width=0.49\textwidth]{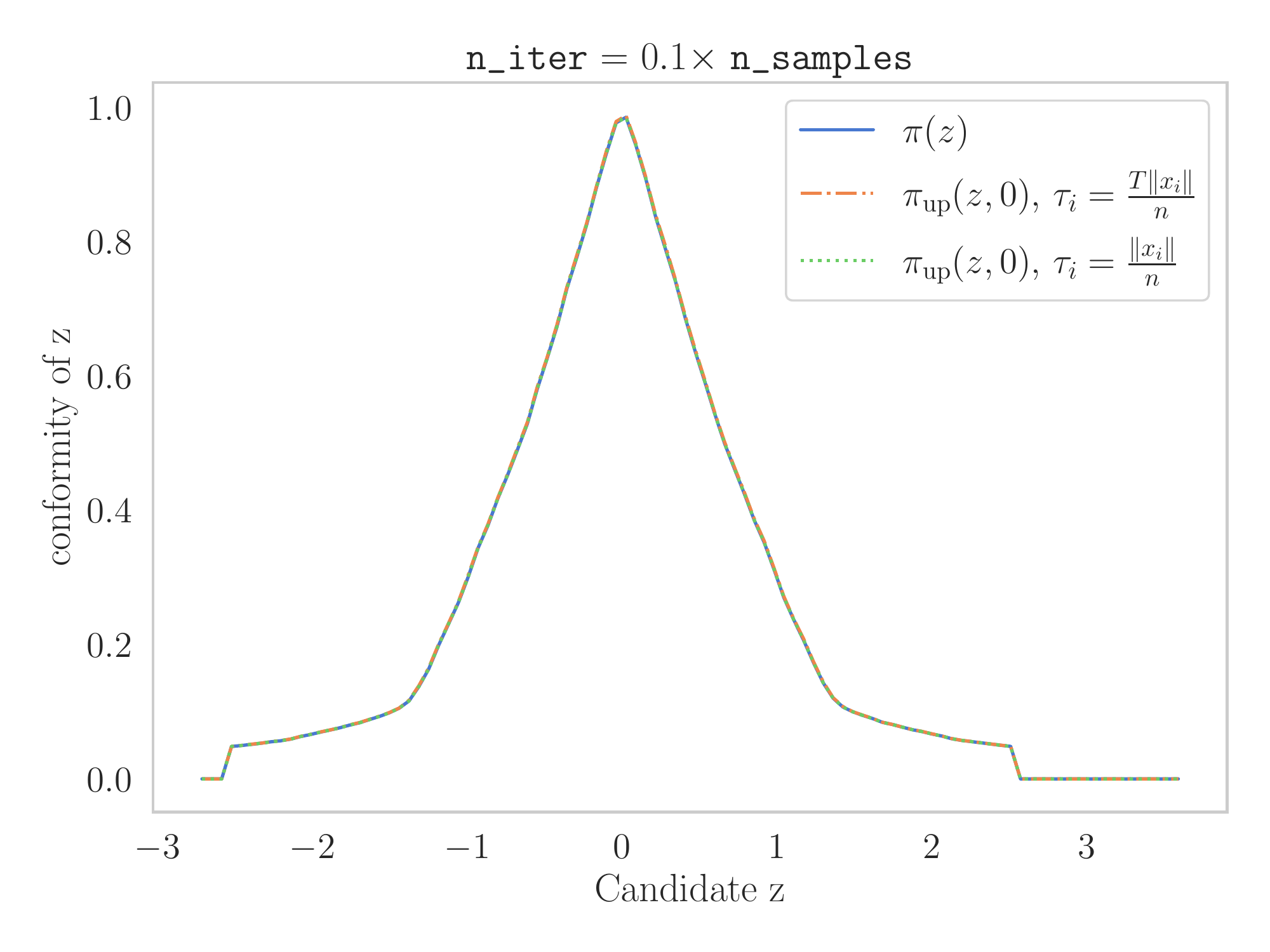}
\includegraphics[width=0.49\textwidth]{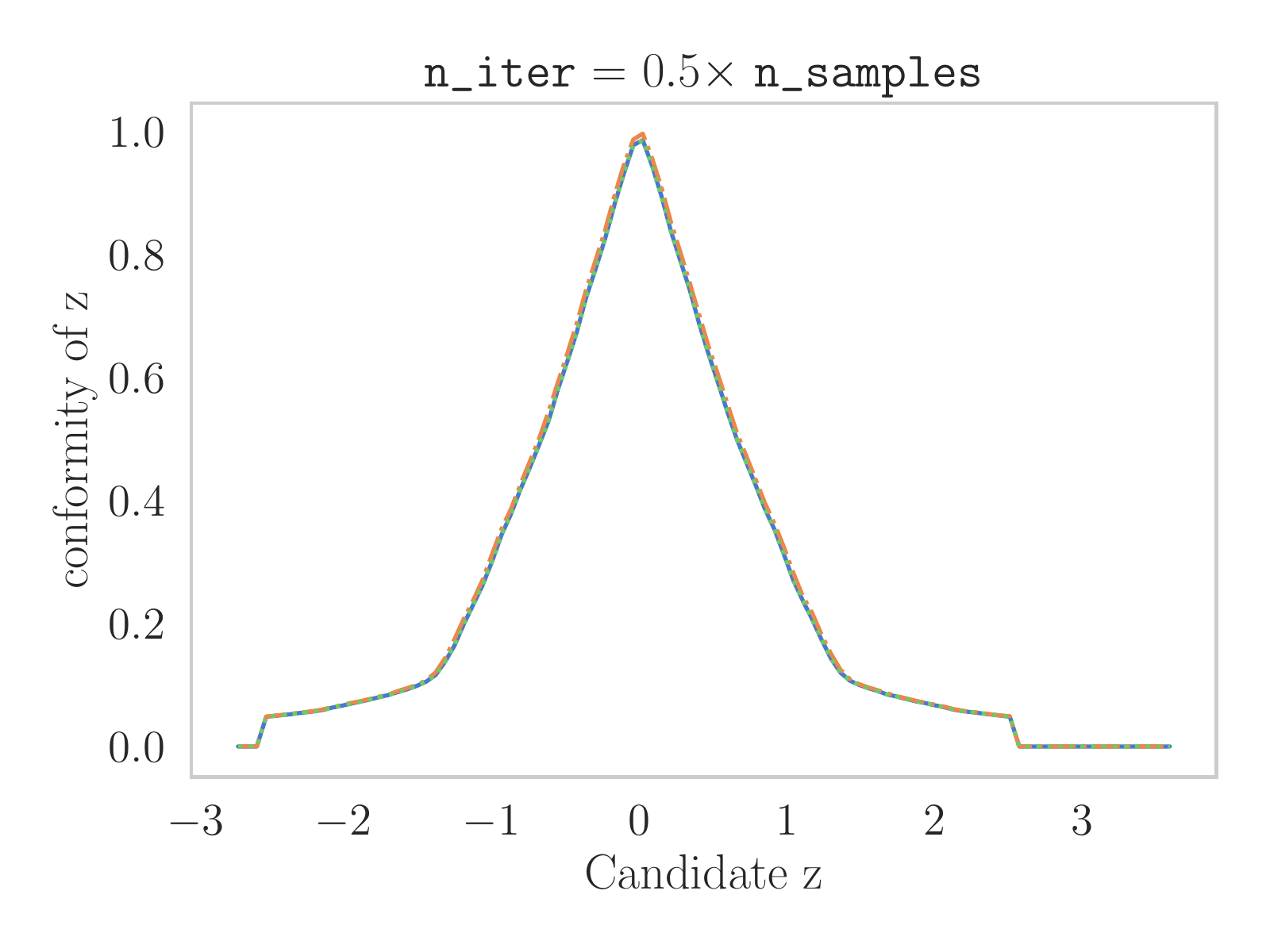}
\includegraphics[width=0.49\textwidth]{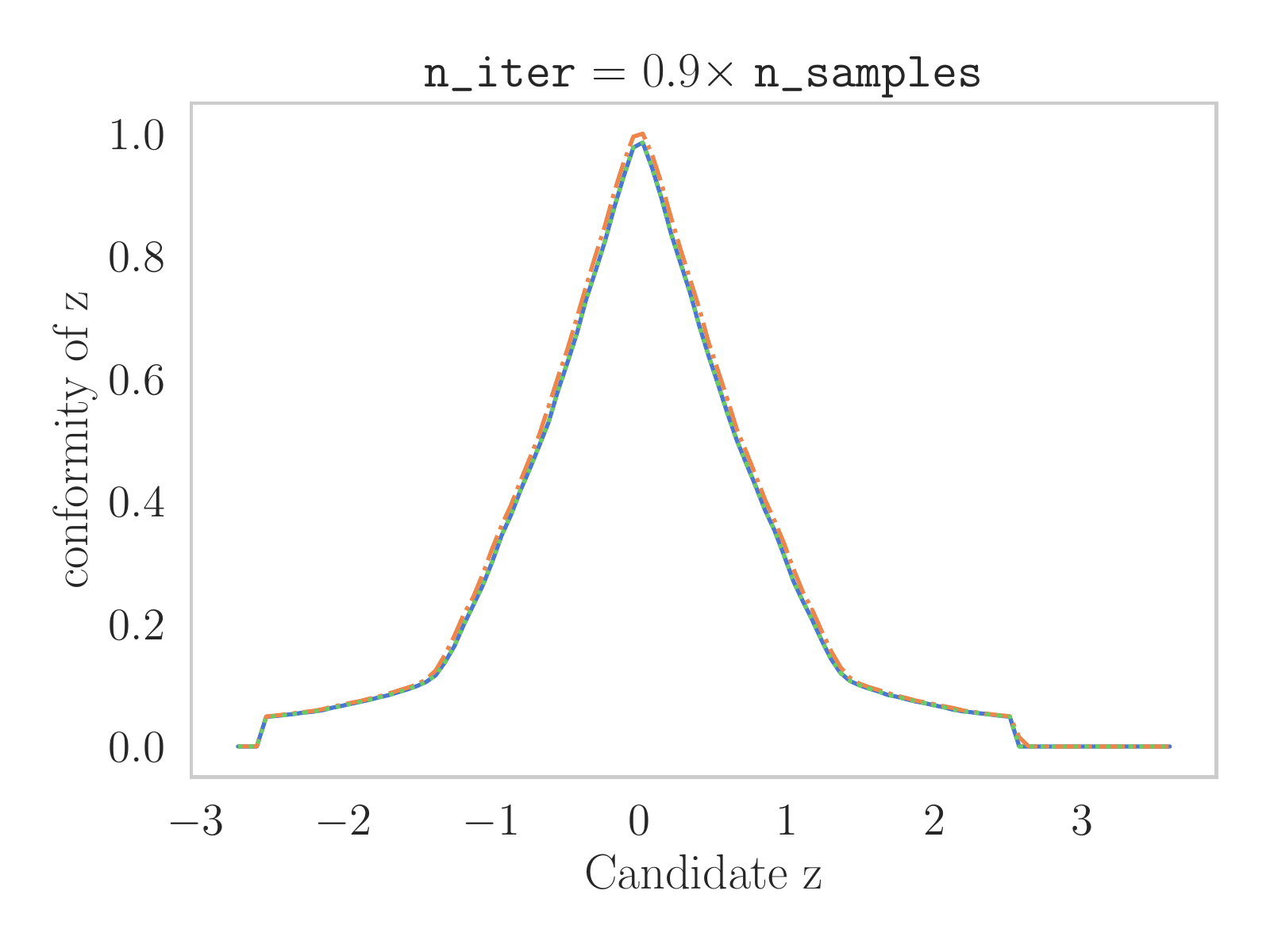}
\includegraphics[width=0.49\textwidth]{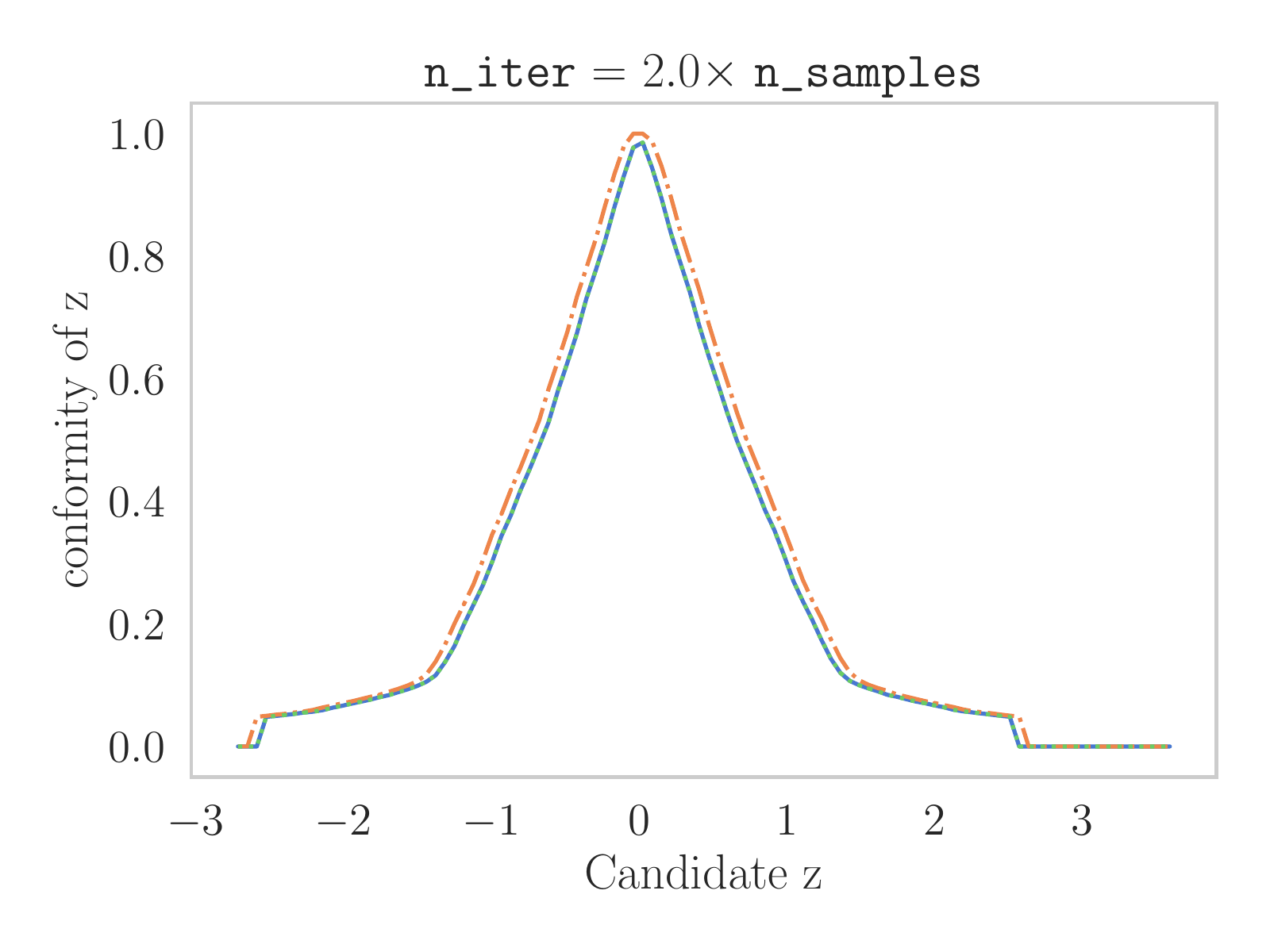}
\caption{Illustration of different conformity functions with respect to a sequence of stability bounds. We observe that by merely staking an order of magnitude $O(1/n)$ as stability bound, gives a good estimate of the conformal prediction set even if the bound is not safe. These experiments are conducted with a Multi-Layer Perceptron regressor on the \texttt{Housingcalifornia} $(20640, 8)$ dataset, trained with $T=\texttt{n\_iter}$ iterations of Stochastic Gradient Descent. \label{MLP_housingcalifornia_different_stability}}
\end{figure*}

\begin{figure*}
\centering
\subfigure[\texttt{Boston} $(506, 13)$]{\includegraphics[width=0.49\textwidth]{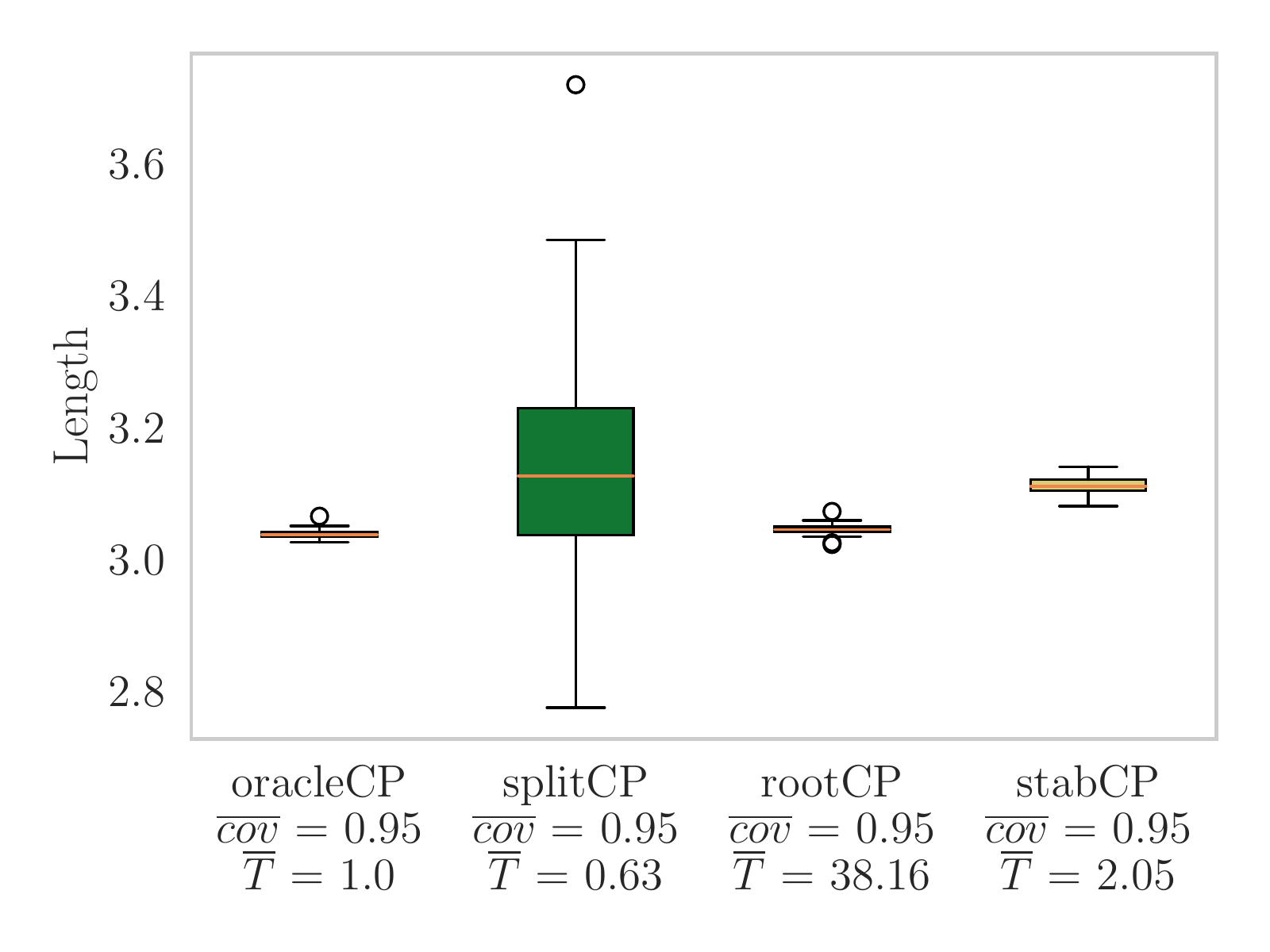}}
\subfigure[\texttt{Diabetes} $(442, 10)$]{\includegraphics[width=0.49\textwidth]{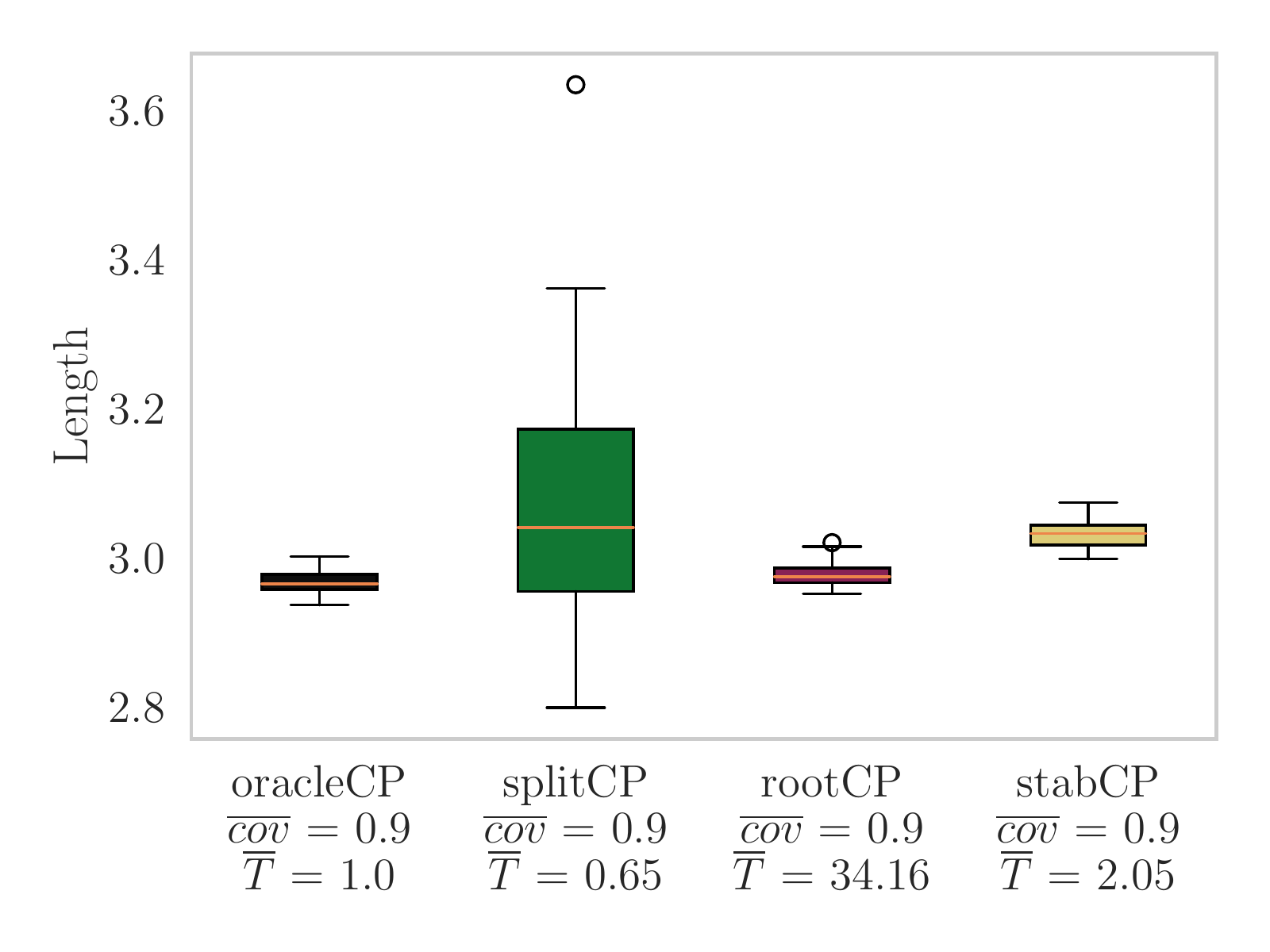}}
\subfigure[\texttt{Housingcalifornia} $(20640, 8)$]{\includegraphics[width=0.49\textwidth]{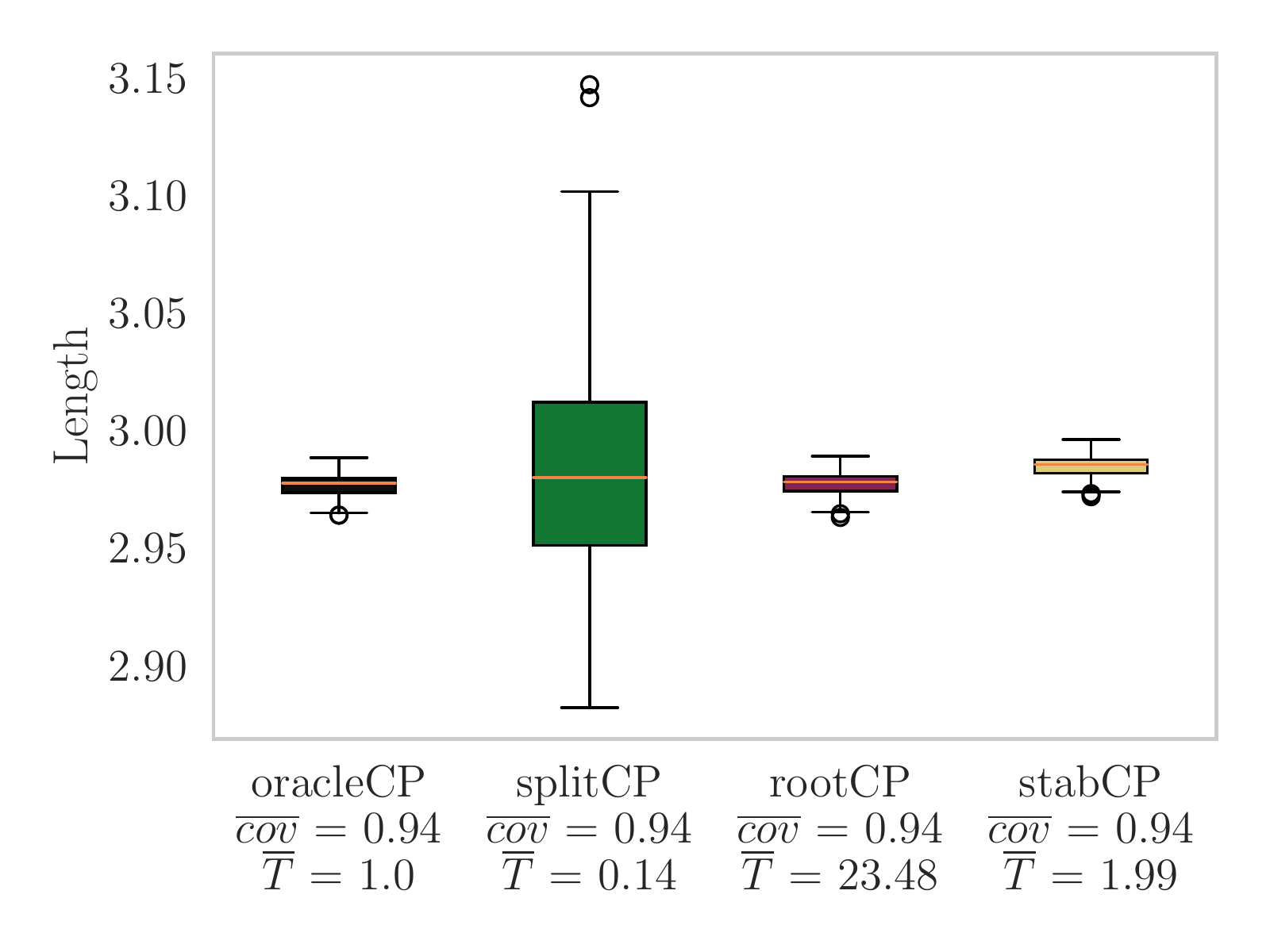}}
\subfigure[\texttt{Friedman1} $(500, 100)$]{\includegraphics[width=0.49\textwidth]{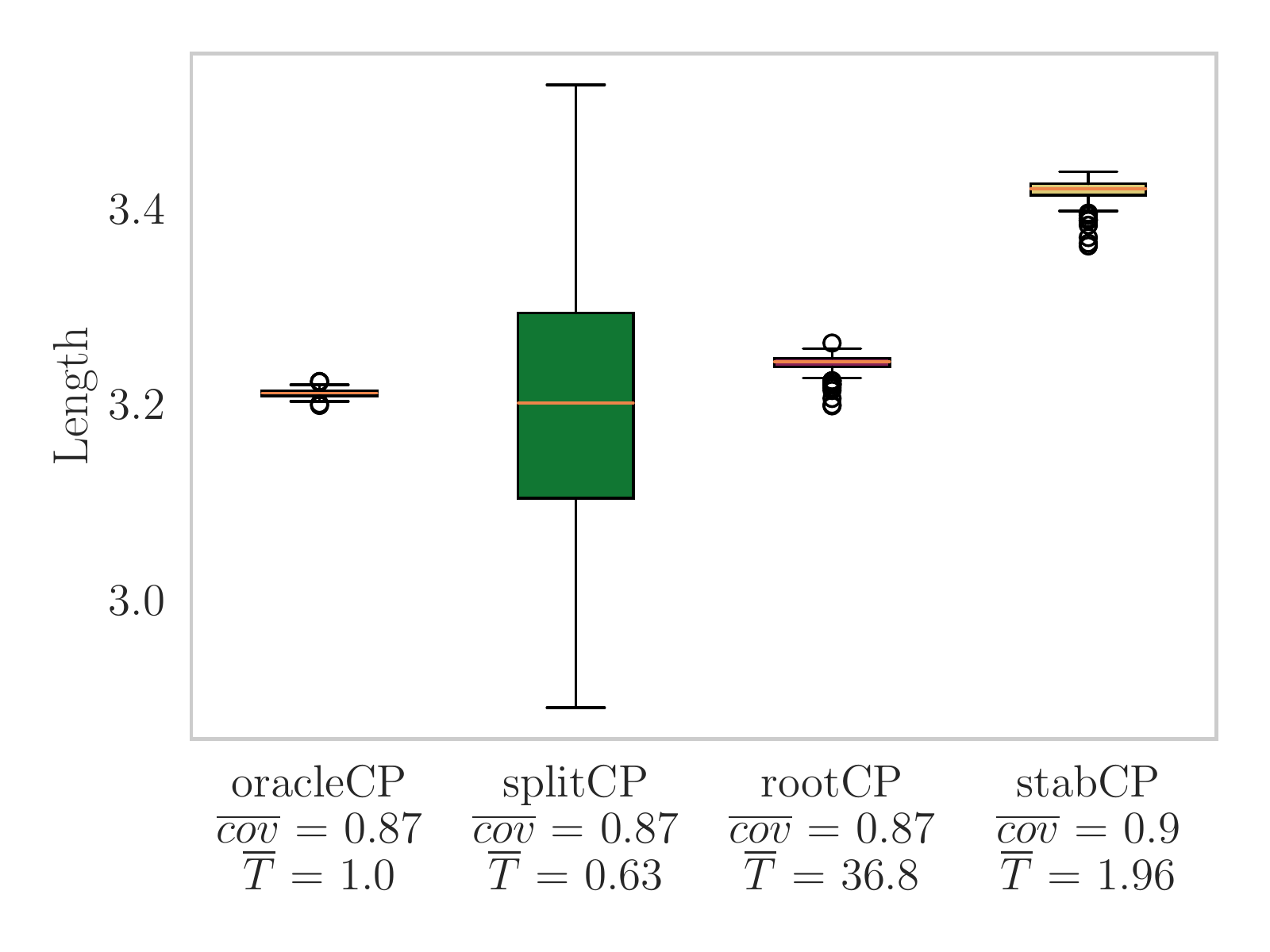}}
\caption{Benchmarking conformal sets for MLP regression models with a ridge regularization on real datasets. The parameter of the model is obtained after $T = n / 10$ iterations of stochastic gradient descent. For \texttt{stabCP}, we use a stability bound estimate $\tau_i = T \norm{x_i} / (n + 1)$. We display the lengths of the confidence sets over $100$ random permutation of the data. We denoted $\overline{cov}$ the average coverage and $\overline{T}$ the average computational time normalized with the average time for computing \texttt{oracleCP} which requires a single full data model fit.\label{fig:MLP_benchmarks}}
\end{figure*}

\begin{figure*}
    \centering
    \subfigure[\texttt{Boston} $(506, 13)$]{\includegraphics[width=0.49\textwidth]{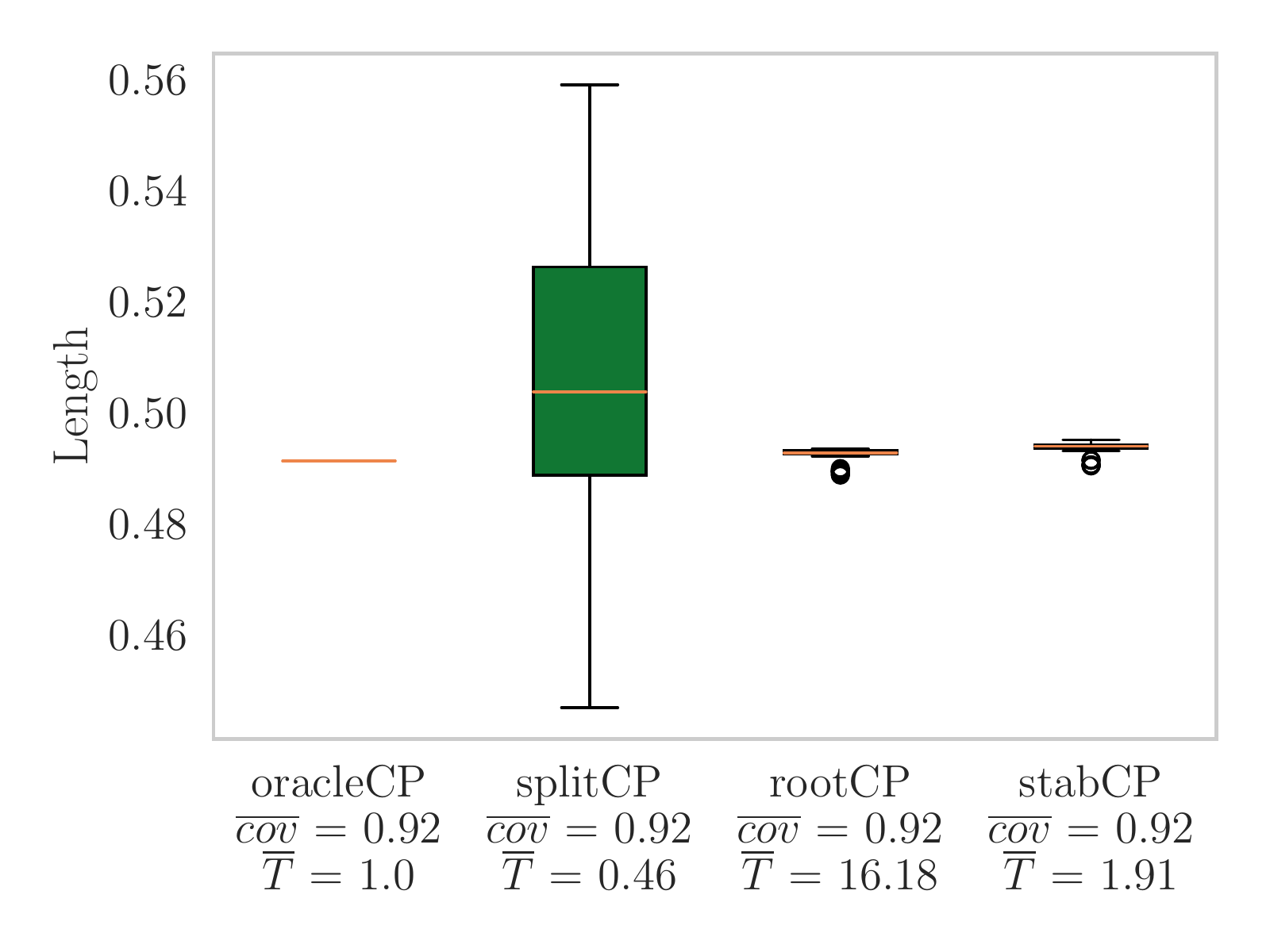}}
    \subfigure[\texttt{Diabetes} $(442, 10)$]{\includegraphics[width=0.49\textwidth]{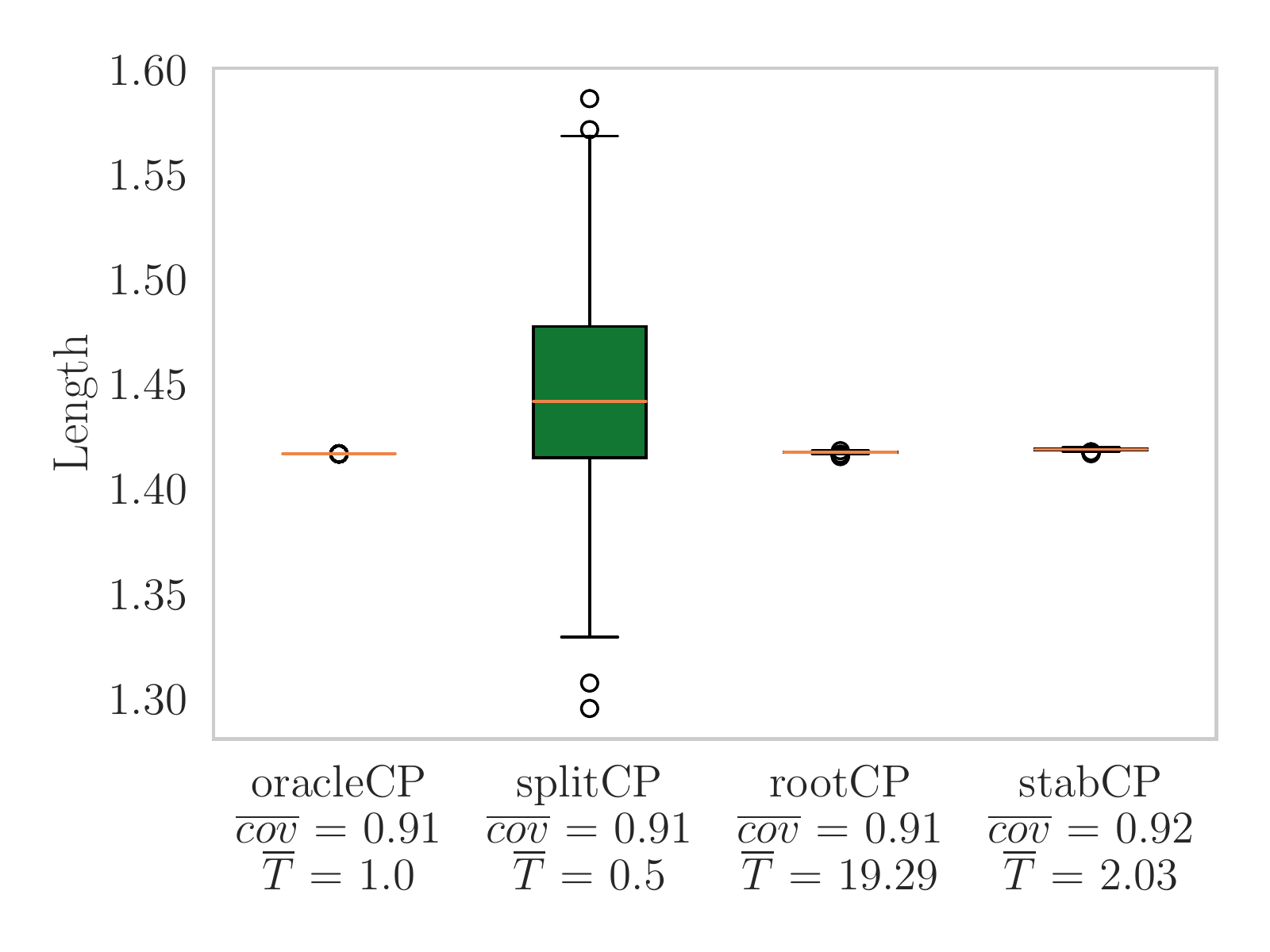}}
    \subfigure[\texttt{Housingcalifornia} $(20640, 8)$]{\includegraphics[width=0.49\textwidth]{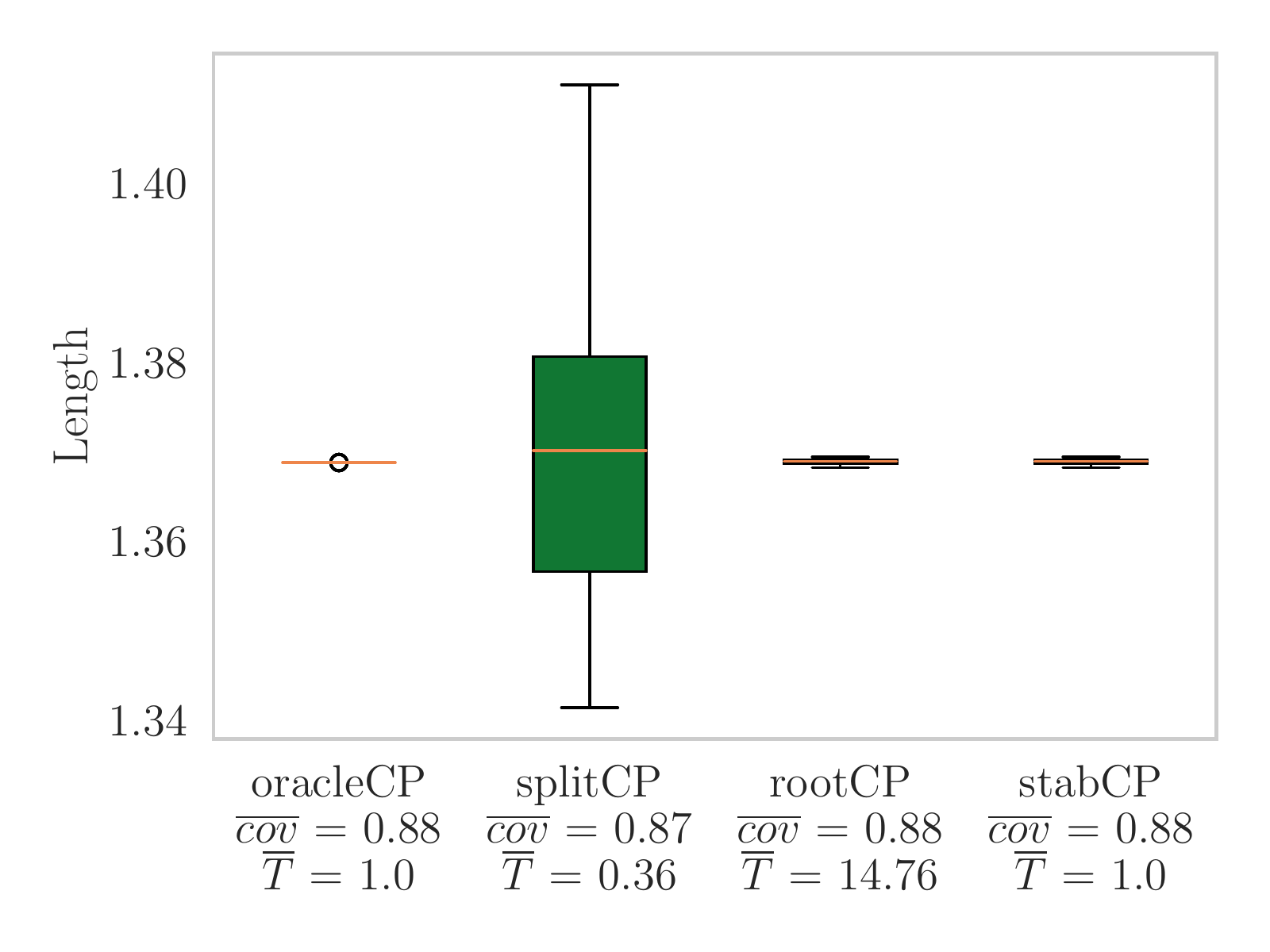}}
    \subfigure[\texttt{Friedman1} $(500, 100)$]{\includegraphics[width=0.49\textwidth]{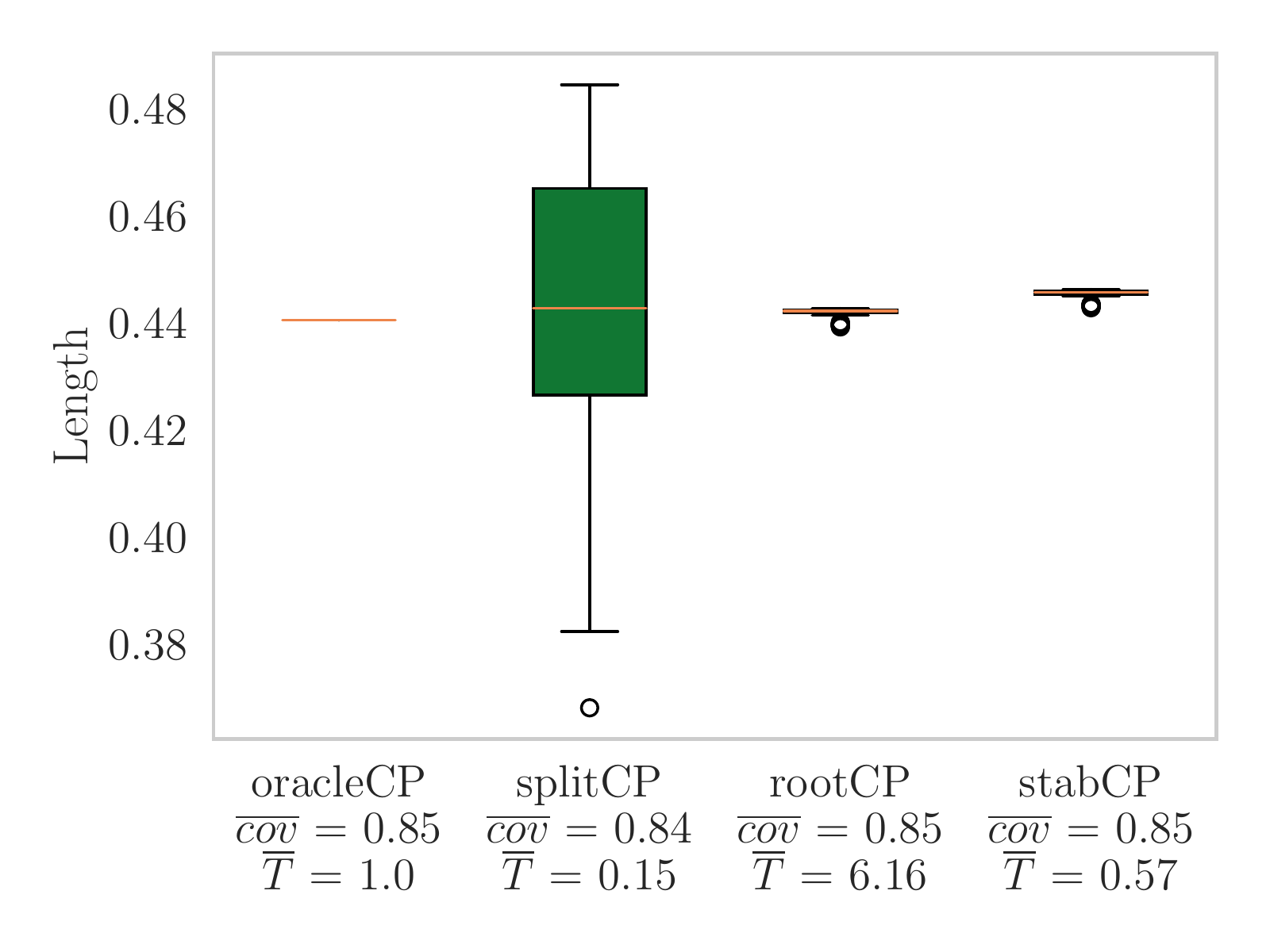}}
    \caption{Benchmarking conformal sets for Gradient Boosting regression models with a ridge regularization on real datasets. For \texttt{stabCP}, we use a stability bound estimate $\tau_i = \norm{x_i} / (n + 1)$. We display the lengths of the confidence sets over $100$ random permutation of the data. We denoted $\overline{cov}$ the average coverage and $\overline{T}$ the average computational time normalized with the average time for computing \texttt{oracleCP} which requires a single full data model fit.\label{fig:1_GradientBoosting_benchmarks}}
\end{figure*}

\begin{figure*}
    \centering
    \subfigure[\texttt{Boston} $(506, 13)$]{\includegraphics[width=0.49\textwidth]{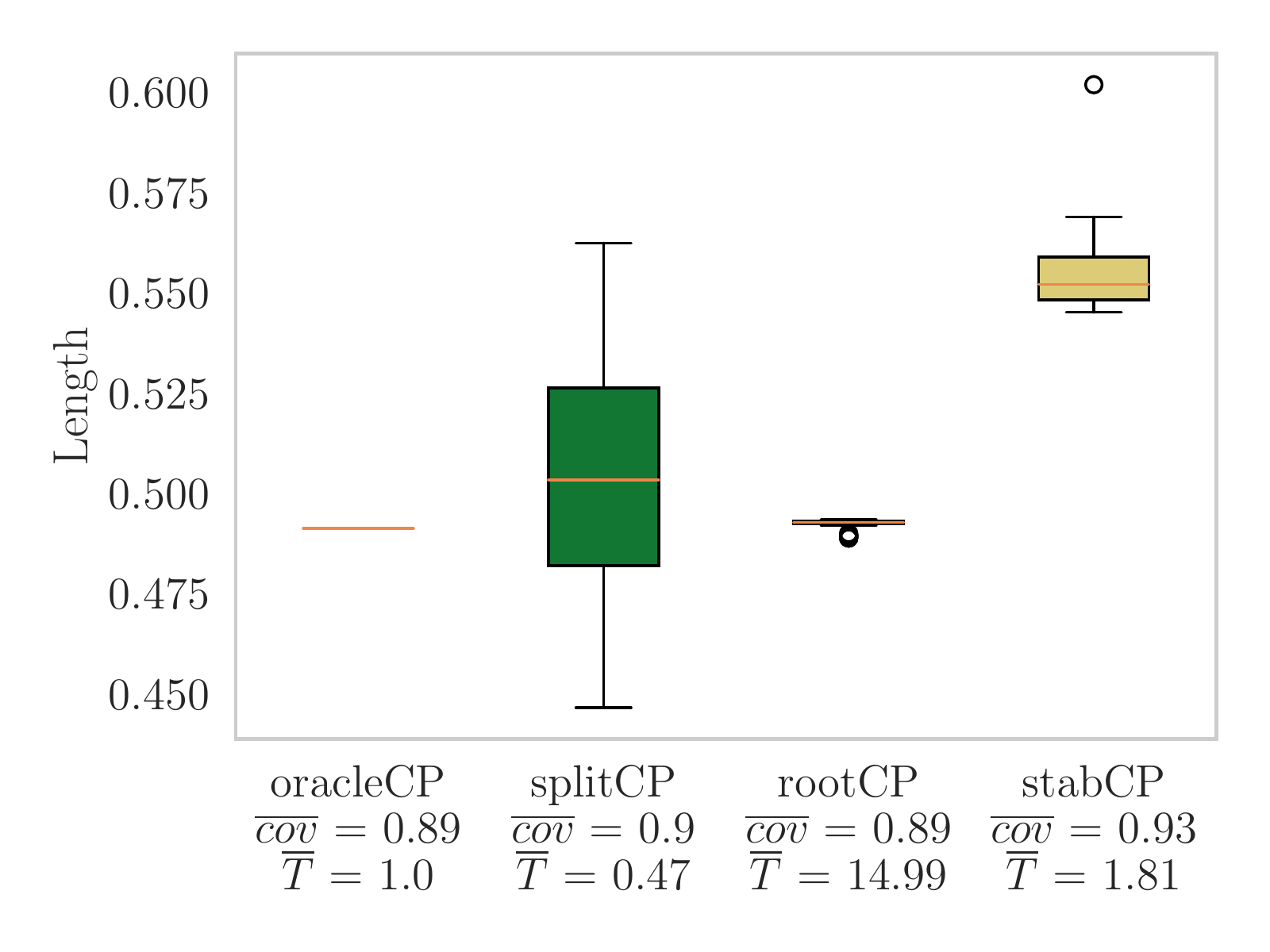}}
    \subfigure[\texttt{Diabetes} $(442, 10)$]{\includegraphics[width=0.49\textwidth]{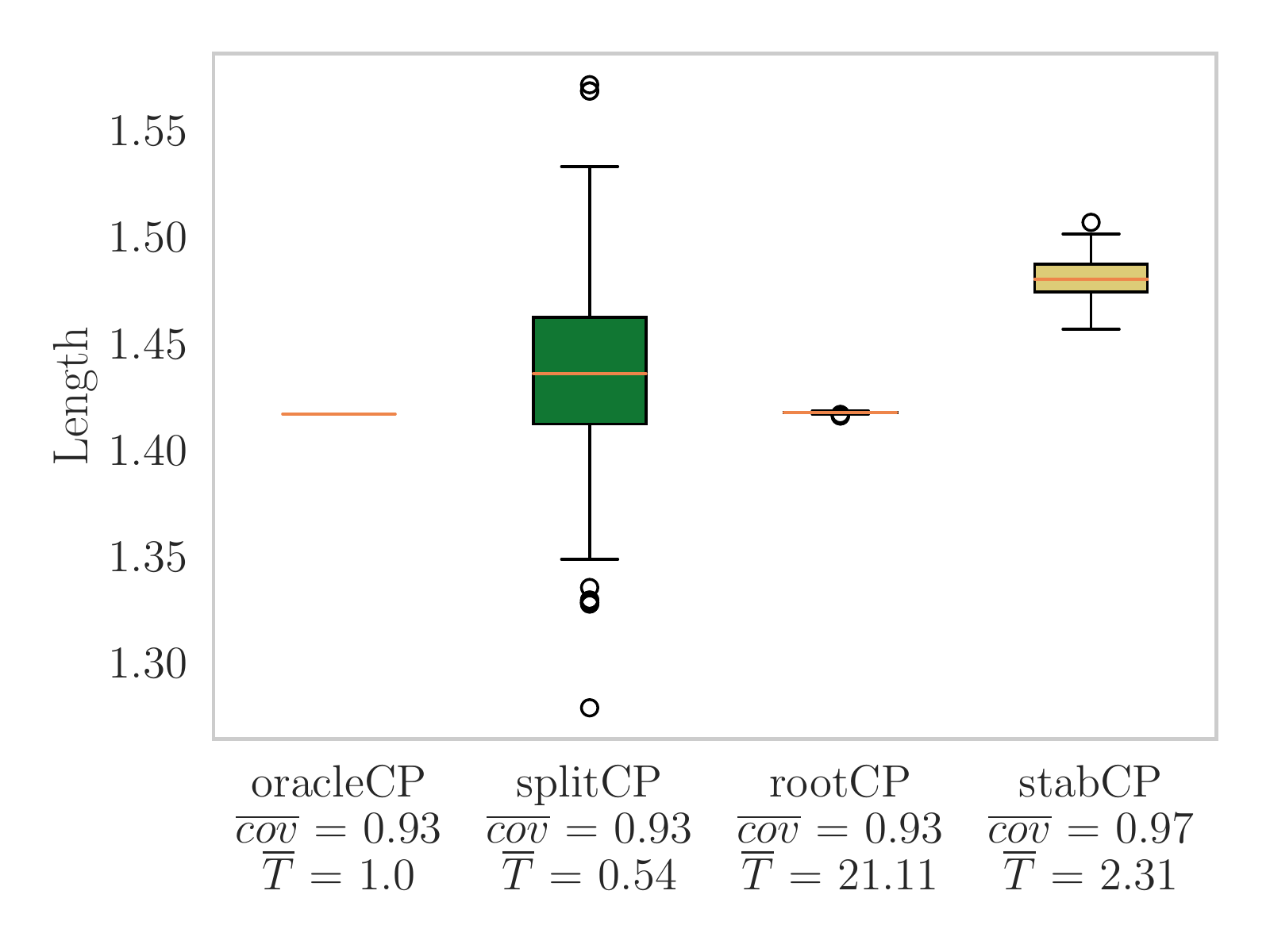}}
    \subfigure[\texttt{Housingcalifornia} $(20640, 8)$]{\includegraphics[width=0.49\textwidth]{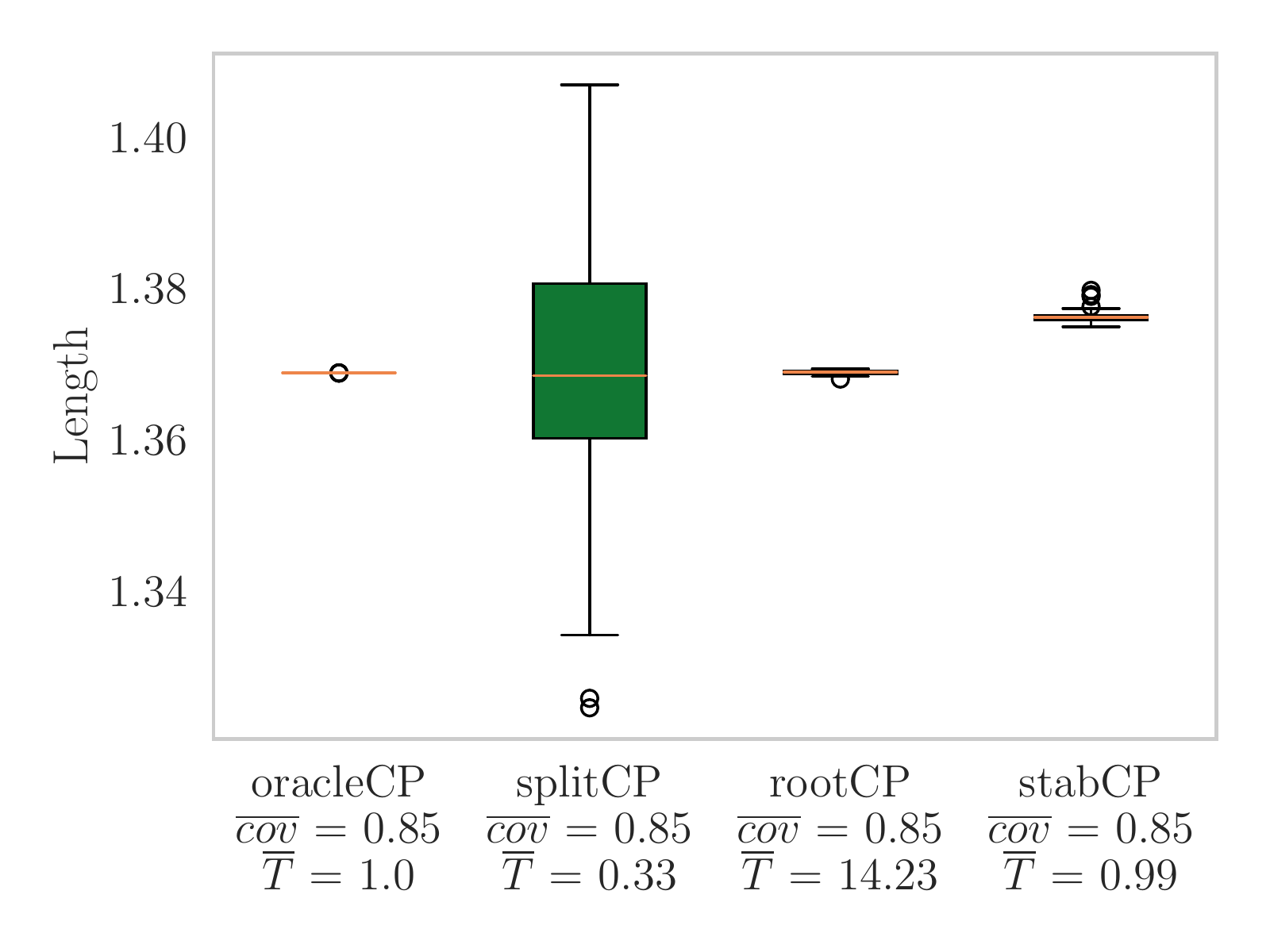}}
    \subfigure[\texttt{Friedman1} $(500, 100)$]{\includegraphics[width=0.49\textwidth]{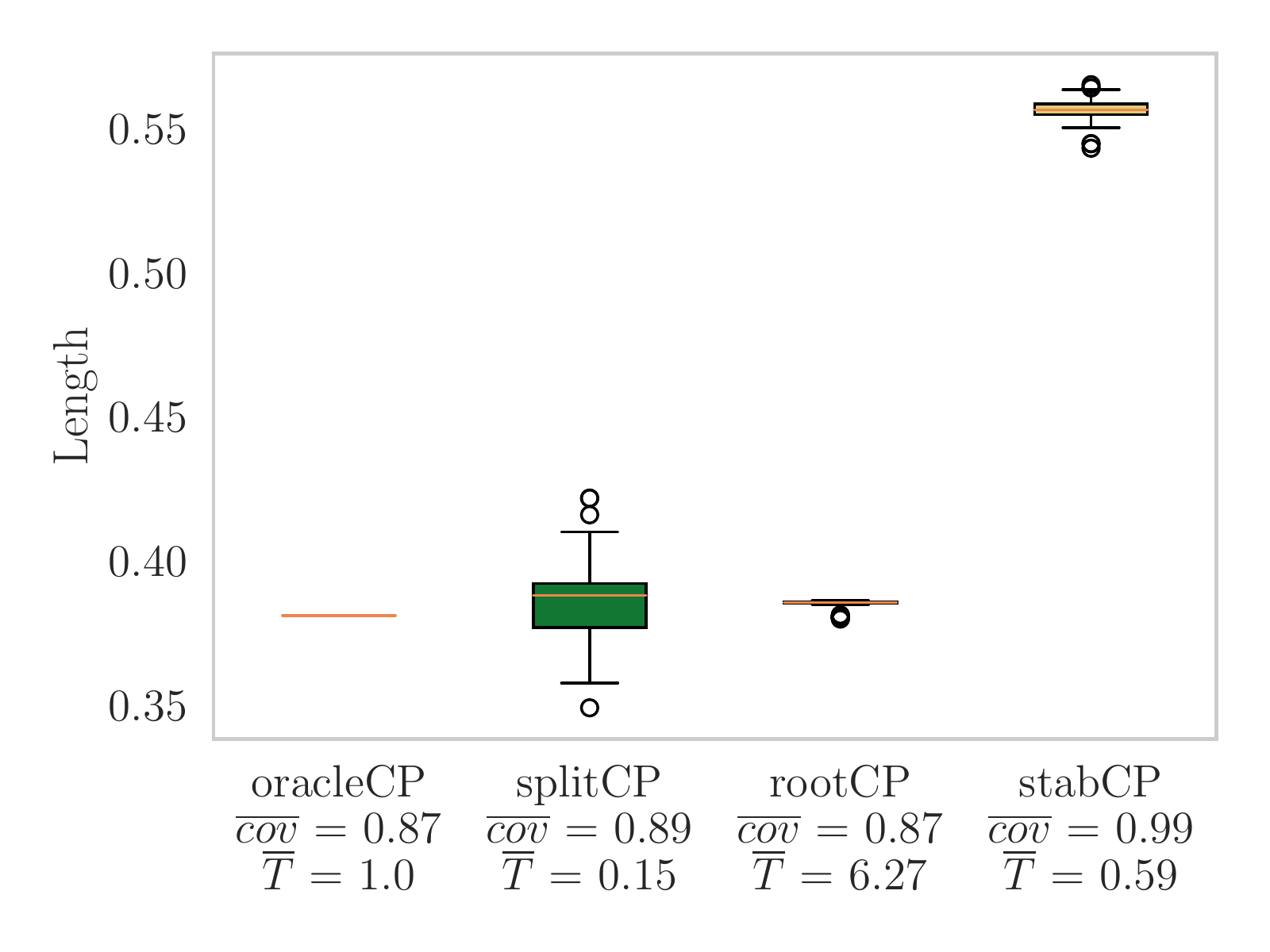}}
    \caption{Benchmarking conformal sets for Gradient Boosting regression models with a ridge regularization on real datasets. For \texttt{stabCP}, we use a rough stability bound estimate $\tau_i \approx \norm{x_i} / 10$. We display the lengths of the confidence sets over $100$ random permutation of the data. We denoted $\overline{cov}$ the average coverage and $\overline{T}$ the average computational time normalized with the average time for computing \texttt{oracleCP} which requires a single full data model fit. This example shows that for unstable models such as decision trees, a coarse estimation of the stability bound can result in an overestimation of the confidence interval, which is a notable limitation of the proposed method.\label{fig:GradientBoosting_benchmarks}}
\end{figure*}